\documentclass{article}

\PassOptionsToPackage{numbers, compress}{natbib}

\usepackage[final]{nips_2018}

\usepackage{amssymb}
\usepackage{amsmath}
\usepackage{mathtools}
\usepackage{subfigure}
\usepackage{dsfont}
\usepackage[colorlinks]{hyperref}

\usepackage[utf8]{inputenc} 
\usepackage[T1]{fontenc}    
\usepackage{hyperref}       
\usepackage{url}            
\usepackage{booktabs}       
\usepackage{amsfonts}       
\usepackage{nicefrac}       
\usepackage{microtype}      

\usepackage{algorithm}
\usepackage{algorithmic}
\usepackage{wrapfig} 
\usepackage[capitalize,noabbrev]{cleveref}

\title{
Stochastic Chebyshev Gradient Descent \\
for Spectral Optimization
}

\author{
  Insu Han\textsuperscript{\rm 1}, Haim Avron\textsuperscript{\rm 2} and Jinwoo Shin\textsuperscript{\rm 1,3} \\
  \textsuperscript{\rm 1}School of Electrical Engineering, Korea Advanced Institute of Science and Technology\\
  \textsuperscript{\rm 2}Department of Applied Mathematics, Tel Aviv University \\
  \textsuperscript{\rm 3}AItrics\\
  \texttt{\{insu.han,jinwoos\}@kaist.ac.kr \quad haimav@post.tau.ac.il}
}

\newtheorem{proposition}{Proposition}
\newtheorem{definition}{Definition}
\newtheorem{theorem}{Theorem}
\newtheorem{lemma}[theorem]{Lemma}
\newtheorem{corollary}[theorem]{Corollary}

\newenvironment{proof}{\setlength\parindent{0pt}{\bf Proof.    }}{\hfill\rule{2mm}{2mm}}

\newcommand{\tr}[1]{{\tt tr}\left({#1}\right)}
\newcommand{\pn}[1]{p_n\left({#1}\right)}
\newcommand{\wpn}[1]{\widehat{p}_n\left({#1}\right)}
\newcommand{\wpr}[1]{\widehat{p}_r\left({#1}\right)}

\newcommand{\abs}[1]{\left|{#1}\right|}
\newcommand{\norm}[1]{\lVert#1\rVert}
\newcommand{\E}[1]{\mathbf E\left[{#1}\right]}
\newcommand{\Ec}[2]{\mathbf E_{#2}\left[{#1}\right]}

\newcommand{\inner}[1]{\left \langle {#1} \right \rangle}
\newcommand{\vv}{\mathbf{v}}
\newcommand{\ww}{\mathbf{w}}
\newcommand{\yy}{\mathbf{y}}
\newcommand{\SM}{\mathcal{S}^{d \times d}}
\newcommand{\dtheta}{\frac{\partial}{\partial \theta_{i}}}
\newcommand{\normf}[1]{\left\lVert#1\right\rVert_{F}}
\newcommand{\normt}[1]{\left\lVert#1\right\rVert_{2}}
\newcommand{\normnuc}[1]{\left\lVert#1\right\rVert_{\mathtt{nuc}}}
\newcommand{\dA}{\frac{\partial A}{\partial \theta_{i}}}
\newcommand{\C}{\mathcal{C}}
\newcommand{\proj}[1]{\Pi_{\C}\left(#1\right)}

\begin{document}

\maketitle

\begin{abstract}
A large class of machine learning techniques requires the solution of optimization problems involving spectral functions of parametric matrices, e.g. log-determinant and nuclear norm. Unfortunately, computing the gradient of a spectral function is generally of cubic complexity, as such gradient descent methods are rather expensive for optimizing  objectives involving the spectral function.  Thus, one naturally turns to stochastic gradient methods in hope that they will provide a way to reduce or altogether avoid the computation of full gradients. However, here a new challenge appears:  there is no straightforward way to compute unbiased stochastic gradients for spectral functions. In this paper, we develop unbiased stochastic gradients for spectral-sums, an important subclass of spectral functions. Our unbiased stochastic gradients are based on combining randomized trace estimators with stochastic truncation of the Chebyshev expansions. A careful design of the truncation distribution allows us to offer distributions that are variance-optimal, which is crucial for fast and stable convergence of stochastic gradient methods. We further leverage our proposed stochastic gradients to devise stochastic methods for objective functions involving spectral-sums, and rigorously analyze their convergence rate.  The utility of our methods is demonstrated in numerical experiments.

\end{abstract}

\section{Introduction}
A large class of machine learning techniques involves {\em spectral optimization} 
problems of the form,  
\begin{equation}
\label{eq:spectral_opt}
\min_{\theta \in {\cal C}} F(A(\theta)) + g(\theta),
\end{equation}
where $\C$ is some finite-dimensional parameter space, $A$ is a function that maps a parameter vector $\theta$ to a symmetric matrix $A(\theta)$, $F$ is a {\em spectral function} (i.e., a real-valued function on symmetric matrices that depends only on the eigenvalues of the input matrix), and $g:\C \to \mathbb{R}$. Examples include hyperparameter learning in Gaussian process regression with $F(X)=\log\det X$~\cite{rasmussen2004gaussian}, nuclear norm regularization with $F(X)=\tr{X^{1/2}}$~\cite{mohan2012iterative}, phase retrieval with $F(X)=\tr{X}$~\cite{friedlander2016spectral}, and quantum state tomography 
with $F(X)=\tr{X\log X}$~\cite{kolt15optimal}.
In the aforementioned applications, the main difficulty in solving problems of the 
form~\eqref{eq:spectral_opt} is in efficiently addressing the spectral component $F(A(\cdot))$.
While explicit formulas for the gradients of spectral functions can be derived~\cite{lewis1996spectral}, it is typically computationally expensive. For example, for $F(X)=\log \det X$ and $A(\theta)\in \mathbb{R}^{d\times d}$, the exact computation of  
$\nabla_\theta F(A(\theta))$ can take as much as $O(d^3 k)$, where $k$ is the number of parameters in $\theta$. Therefore, it is
desirable to avoid computing, or at the very least reduce the number of times we compute, the gradient of $F(A(\theta))$ exactly.

It is now well appreciated in the machine learning literature that the use of stochastic gradients is effective in alleviating costs
associated with expensive exact gradient computations. Using cheap stochastic gradients, one can avoid computing full gradients
altogether by using Stochastic Gradient Descent (SGD). The cost is, naturally, a reduced rate of convergence. Nevertheless, 
many machine learning applications require only mild suboptimality, in which case cheap iterations often outweigh the 
reduced convergence rate. When nearly optimal solutions are sought, more recent variance reduced methods 
(e.g. SVRG~\cite{johnson2013accelerating}) are effective in reducing the number of full gradient computations to $O(1)$.
For non-convex objectives, the stochastic methods are even more attractive to use as they allow to avoid a bad local optimum.
However, 
closed-form formulas for computing the full gradients of spectral functions do not lead to efficient stochastic gradients in a straightforward manner.

{\bf Contribution.}
In this paper, we propose stochastic methods for solving~\eqref{eq:spectral_opt} when the spectral function $F$ is a 
{\em spectral-sum}. Formally, spectral-sums are spectral functions that can be expressed as $F(X)=\tr{f(X)}$ where 
$f$ is a real-valued function that is lifted to the symmetric matrix domain by applying it to the eigenvalues. 
They constitute an important subclass of spectral functions, e.g., in all of the aforementioned applications of spectral optimization, the spectral function $F$ is a spectral-sum. 

Our algorithms are based on recent {\em biased} estimators for spectral-sums that combine stochastic trace estimation with Chebyshev expansion~\cite{han2016approximating}. The technique used to derive these estimators can also be used to derive 
stochastic estimators for the gradient of spectral-sums (e.g., see \cite{dong2017scalable}), but the resulting estimator is 
biased. To address this issue, we propose an {\em unbiased} estimator for spectral-sums, and use it to derive unbiased 
stochastic gradients. Our unbiased estimator is based on randomly selecting the truncation degree in the Chebyshev expansion, i.e., the truncated polynomial degree is drawn under some distribution. We remark that similar ideas of sampling unbiased polynomials have been studied in the literature, but for different setups~ \cite{broniatowski2014some,lee2015faster,vinck2012estimation,adams2018estimating}, 
and  none of which are suitable for use in our setup.

While deriving unbiased estimators is very useful for ensuring stable convergence of stochastic gradient methods, it is not sufficient: convergence rates of stochastic gradient descent methods depend on the variance of the 
stochastic gradients, and this can be rather large for na\"ive choices of degree distributions. Thus, our main contribution is in establishing a provably optimal degree distribution minimizing the estimators' variances with respect to the Chebyshev series. The proposed distribution gives order-of-magnitude smaller variances compared to other popular ones 
(Figure \ref{fig:var}), which leads to improved convergence of the downstream optimization (Figure~\ref{fig:mc3}).

We leverage our proposed unbiased estimators to design two stochastic gradient descent methods, one using the SGD framework and the other using 
the SVRG one. We rigorously analyze their convergence rates, showing sublinear and linear rate for SGD and SVRG, respectively.
It is important to stress that our fast convergence results crucially depend on the proposed optimal degree distributions. 
Finally, we apply our algorithms to two machine learning tasks that involve spectral optimization: 
matrix completion and learning Gaussian processes.
Our experimental results confirm that the proposed algorithms are significantly faster 
than other competitors under large-scale real-world instances.
In particular, for learning Gaussian process under Szeged humid dataset,  our generic method runs up to six times faster
than the state-of-art method \cite{dong2017scalable} specialized for the purpose.

\section{Preliminaries} \label{sec:prelim}
We denote the family of real symmetric matrices of dimension $d$ by $\SM$.
For $A \in \SM$, we use $\| A \|_{\texttt{mv}}$ to denote the time-complexity of multiplying $A$ with a vector, i.e., 
$\| A\|_{\texttt{mv}}= O(d^2)$. For some
structured matrices, e.g. low-rank, sparse or Toeplitz matrices, it is possible 
to have $\| A\|_{\texttt{mv}}= o(d^2)$.

\subsection{Chebyshev expansion}
Let $f: \mathbb R \rightarrow \mathbb R$ be an analytic function on $[a,b]$ for $a,b \in \mathbb{R}$. 
Then, the Chebyshev series of $f$ is given by
\begin{align*}
f(x) = \sum_{j=0}^\infty b_j T_j\left( \frac2{b-a}x - \frac{b+a}{b-a}\right),  \ \ 
b_j = \frac{2 - \mathds{1}_{j=0}}{\pi}\int_{-1}^1 \frac{f\left(\frac{b-a}2 x + \frac{b+a}{2}\right) T_j(x)}{\sqrt{1 - x^2}} dx.
\end{align*}
In the above, 
$\mathds{1}_{j=0}= 1$ if $j=0$ and $0$ otherwise
and
$T_j(x)$ is the Chebyshev polynomial (of the first kind) of degree $j$. 
An important property of the Chebyshev polynomials is the following recursive formula:
$T_{j+1}(x) = 2x T_j(x) - T_{j-1}(x),~ T_1(x)=x,~ T_0(x)=1$.
The Chebyshev series can be used to approximate $f(x)$ via 
simply truncating the higher order terms, i.e.,
$
f(x) \approx p_n(x) := 
\sum_{j=0}^n b_j T_j(\frac{2}{b-a}x-\frac{b+a}{b-a}).
$
We call $p_n(x)$ the {\em truncated Chebyhshev series} of degree $n$.
For analytic functions, the approximation error (in the uniform norm) is known to decay exponentially \cite{trefethen2013approximation}.
Specifically, if $f$ is analytic with $\abs{f(\frac{b-a}2 z + \frac{b+a}{2})} \leq U$ for some $U>0$ in the region bounded by the ellipse with foci $+1,-1$ and sum of major and minor semi-axis lengths equals to $\rho > 1$, then
\begin{align}
\abs{b_j} \leq \frac{2U}{\rho^j}, \quad \forall~j \geq 0, \qquad
\sup_{x\in [a,b]} \abs{f(x) - p_n(x)} \leq \frac{4 U}{\left( \rho - 1\right) \rho^n }. \label{eq:decayrate}
\end{align}

\subsection{Spectral-sums and their Chebyshev approximations}
Given a matrix $A \in \mathcal{S}^{d \times d}$ and a function $f : \mathbb{R} \rightarrow \mathbb{R}$, 
the {\em spectral-sum} of $A$ with respect to $f$ is 
\begin{align*}\Sigma_f(A):= \tr{f(A)}=\sum_{i=1}^{d} f(\lambda_i),\end{align*} where
$\tr{\cdot}$ is the matrix trace and $\lambda_1,\lambda_2,\dots,\lambda_d$ are the eigenvalues of $A$.
Spectral-sums constitute an important subclass of spectral functions, and many applications of spectral optimization
involve spectral-sums. This is fortunate since spectral-sums can be well approximated using Chebyshev approximations. 

For a general $f$, one needs all eigenvalues to compute $\Sigma_f(A)$, while for some 
functions, simpler types of decomposition might suffice (e.g., $\log \det A = \Sigma_{\log}(A)$ can be computed
using the Cholesky decomposition). Therefore, the general complexity 
of computing spectral-sums is $O(d^3)$, which is clearly not feasible when $d$ is very large,
as is common in many machine learning applications. 
Hence, it is not surprising that recent literature proposed methods to approximate the large-scale spectral-sums, e.g.,
\cite{han2016approximating} recently 
 suggested a fast randomized algorithm
for approximating spectral-sums 
based on Chebyshev series and Monte-Carlo trace estimators (i.e., Hutchinson's method~\cite{hutchinson1989stochastic}): 
\begin{align} \label{eq:est1}
\Sigma_f(A)&=\tr{f(A)} \approx \tr{p_n(A)} = \mathbf{E}_{\vv}\left[\vv^{\top} p_n(A) \vv\right] 
\approx \frac1{M} \sum_{k=1}^{M} \vv^{(k)\top} \left(\sum_{j=0}^n b_j \ww_j^{(k)}\right)
\end{align}
where $\ww_{j+1}^{(k)} = 2\left(\frac{2}{b-a}A -\frac{b+a}{b-a}I \right)\ww_{j}^{(k)} - \ww_{j-1}^{(k)}, \ \ww_1^{(k)} = \left(\frac{2}{b-a}A -\frac{b+a}{b-a}I \right) \vv, \ \ \ww_0^{(k)} = \vv^{(k)},$ and $\{\vv^{(k)}\}_{k=1}^{M}$ are Rademacher random vectors, i.e.,
each coordinate of $\vv^{(k)}$ is an i.i.d.\ random variable in $\{-1,1\}$ with
equal probability $1/2$ \cite{hutchinson1989stochastic,avron2011randomized,roosta2015improved}.
The approximation~\eqref{eq:est1} can be computed using only matrix-vector multiplications,
vector-vector inner-products and vector-vector additions $O(Mn)$ times each.
Thus, the time-complexity becomes $O(Mn\|A\|_{\texttt{mv}}+Mnd)=O(Mn\|A\|_{\texttt{mv}})$.
In particular, when $Mn\ll d$ and $\|A\|_{\texttt{mv}} = o(d^2)$, 
the cost can be significantly cheaper than $O(d^3)$ of exact computation.
We further note that to apply the approximation \eqref{eq:est1}, 
a bound on the eigenvalues is necessary.
For an upper bound, 
one can use fast power methods \cite{davidson1975iterative}; 
this does not hurt the total algorithm complexity (see \cite{han2015large}).
A lower bound can be encforced by substituting $A$ with 
The lower bound can typically be ensured $A + \varepsilon I$ for some small $\varepsilon > 0$.
We use these techniques in our numerical experiments.

We remark that one may consider other polynomial approximation schemes, e.g. Taylor,  
but 
we focus on the Chebyshev approximations 
since they are nearly optimal in approximation among polynomial series \cite{mason2002chebyshev}.
Another recently suggested powerful technique is {\em stochastic Lanczos quadrature}~\cite{ubaru2017fast},
however it is not suitable for our needs (our bias removal technique is not applicable for it).

\section{Stochastic Chebyshev gradients of spectral-sums} \label{sec:main1}
Our main goal is to develop scalable methods for solving the following optimization problem:
\begin{align} \label{eq:specopt}
\min_{\theta \in \C \subseteq \mathbb{R}^{d^{\prime}}} \Sigma_f(A(\theta)) + g(\theta), 
\end{align}
where {$\C \subseteq \mathbb{R}^{d^\prime}$ is a non-empty, closed and convex domain}, $A: \mathbb{R}^{d^\prime} \rightarrow \SM$ 
is a function of parameter $\theta =[\theta_i]\in \mathbb{R}^{d^{\prime}}$
and $g:\mathbb{R}^{d^{\prime}} \rightarrow \mathbb R$ is some function
whose
derivative with respect to any
parameter $\theta$
is computationally easy to obtain.
Gradient-descent type methods 
are natural candidates for tackling such problems. However, while it is usually possible to compute the gradient of $\Sigma_f(A(\theta))$, this is typically very expensive.  
Thus, we turn to stochastic methods, like ({projected}) SGD \cite{bottou2010large,zinkevich2003online} and SVRG \cite{johnson2013accelerating,xiao2014proximal}. In order to apply stochastic methods, one needs unbiased estimators of the gradient. The goal of this section is to propose a computationally efficient method to generate unbiased stochastic gradients of small variance for $\Sigma_f(A(\theta))$.

\subsection{Stochastic Chebyshev gradients}
{\bf Biased stochastic gradients.} 
We begin by observing that if $f$ is a polynomial itself or the Chebyshev approximation is exact, i.e.,
$f(x)=p_n(x) = \sum^n_{j=0} b_j T_j(\frac{2}{b-a}x-\frac{b+a}{b-a})$, we have
\begin{align} \label{eq:estder}
\dtheta \Sigma_{p_n}(A)
&= \dtheta \tr{p_n(A)} = \dtheta \mathbf{E}_{\vv}\left[\vv^\top p_n(A) \vv\right] 
= \mathbf{E}_{\vv}\left[\dtheta \vv^\top p_n(A) \vv\right] \nonumber\\
&\approx  \frac1M \sum_{k=1}^M \dtheta \vv^{(k)\top} p_n(A) \vv^{(k)} 
= \frac1M \sum_{k=1}^M \vv^{(k)\top} \left( \sum_{j=0}^n b_j \frac{\partial \ww_j^{(k)}}{\partial \theta_{i}}\right)\footnotemark,
\end{align} \footnotetext{We assume that all partial  derivatives $\partial A_{j,k} /\partial \theta_i$ for $j,k=1,\dots,d, i=1,\dots,d^\prime$ exist and are continuous.}
where $\{\vv^{(k)}\}_{k=1}^M$ are i.i.d. Rademacher random vectors 
and ${\partial \mathbf{w}^{(k)}_j}/{\partial \theta_i}$ are given by the following recursive formula:
\begin{align} \label{eq:updatew}
\textstyle
&\frac{\partial \ww_{j+1}^{(k)}}{\partial \theta_{i}}
= \frac4{b-a} \frac{\partial A}{\partial \theta_{i}}\ww_j^{(k)} 
+ 2 \widetilde{A} \frac{\partial \ww_j^{(k)}}{\partial \theta_{i}} - \frac{\partial \ww_{j-1}^{(k)}}{\partial \theta_{i}},
\ \ 
\frac{\partial \ww_{1}^{(k)}}{\partial \theta_{i}} = \frac2{b-a}\frac{\partial A}{\partial\theta_{i}}\vv^{(k)},
\ \ 
\frac{\partial \ww_{0}^{(k)}}{\partial \theta_{i}} = \mathbf{0},
\end{align}
and $\widetilde{A} = \frac2{b-a}A - \frac{b+a}{b-a}I$.
We note that in order to compute \eqref{eq:updatew} only matrix-vector products with $A$ and $\partial A / \partial \theta_i$ are needed. Thus, stochastic gradients of spectral-sums involving polynomials of degree $n$ can be computed in $O(Mn ( \|A\|_{\mathtt{mv}}\ d' + \sum_{i=1}^{d'} \| \frac{\partial A}{\partial \theta_i}\|_{\mathtt{mv}}))$.
As we shall see in Section \ref{sec:exp}, 
the complexity can be further reduced in certain cases.
The above estimator can be leveraged to approximate gradients for spectral-sums of analytic functions 
via the truncated Chebyshev series: $\nabla_\theta \Sigma_f(A(\theta)) \approx \nabla_\theta \Sigma_{p_n}(A(\theta))$. 
Indeed, \cite{dong2017scalable} recently explored this in the context of Gaussian process kernel learning.
However, if $f$ is not a polynomial, the truncated Chebyshev series $p_n$ is not equal to $f$, so the above estimator is biased, i.e. $\nabla_{\theta} \Sigma_f(A)\neq \mathbf{E}[\nabla_\theta \mathbf{v}^\top p_n(A) \mathbf{v}]$.
The biased stochastic gradients might hurt iterative stochastic optimization as
biased errors accumulate over iterations. 

{\bf Unbiased stochastic gradients.} The estimators \eqref{eq:est1} and \eqref{eq:estder}
are biased since they approximate an analytic function $f$ via a polynomial $p_n$ of fixed degree.
Unless $f$ is a polynomial itself, there exists an $x_0$ (usually uncountably many) 
for which $f(x_0)\neq p_n(x_0)$, so if $A$ has an eigenvalue at $x_0$ we have $\Sigma_f (A) \neq \Sigma_{p_n} (A)$.
Thus, one cannot hope that the estimator \eqref{eq:est1}, let alone the gradient estimator \eqref{eq:estder}, to
be unbiased for {\em all} matrices $A$.
To avoid deterministic truncation errors, we simply randomize the degree, i.e.,
design some distribution ${\cal D}$ on polynomials such that for every $x$ we have
$\Ec{p(x)}{p\sim{\cal D}} = f(x)$. 
This guarantees 
$\Ec{\tr{p(A)}}{p\sim{\cal D}} = \Sigma_f (A)$ 
from the linearity of expectation.

We propose to build such a distribution on polynomials by using truncated Chebyshev expansions
where the truncation degree is stochastic.
Let $\{q_i\}^\infty_{i=0} \subseteq [0,1]$ be a set of numbers such that 
$\sum_{i=0}^\infty q_i = 1$ and $\sum_{i=r}^\infty q_i>0$ for all $r\geq0$.
We now define for $r=0,1,\dots$
\begin{align} \label{eq:estimator}
\wpr{x} \coloneqq \sum_{j=0}^r \frac{b_j}{1 - \sum_{i=0}^{j-1} q_i} T_j\left( \frac{2}{b-a}x - \frac{b+a}{b-a} \right).
\end{align}
Note that $\wpr{x}$ can be obtained from  $p_r(x)$ by re-weighting each
coefficient according to $\{q_i\}^\infty_{i=0}$. Next, let $n$ be a random variable
taking non-negative integer values, and defined according to
$\Pr(n=r)=q_r$. Under certain conditions on $\{q_i\}$, 
$\wpn{\cdot}$ can be used to derive unbiased estimators of $\Sigma_f(A)$ and $\nabla_\theta\Sigma_f(A)$ as stated in the following lemma.
\begin{lemma}\label{lmm:unbiased}
Suppose that $f$ is an analytic function  
and $\widehat{p}_n$ is the randomized Chebyshev series of $f$ in \eqref{eq:estimator}.
Assume that the entries of $A$ are differentiable for $\theta \in {\cal C}'$, where ${\cal C}'$ is
an open set containing ${\cal C}$, and that for $a,b \in \mathbb{R}$ all the eigenvalues of $A(\theta)$ for $\theta \in {\cal C}'$ are in $[a,b]$.
For any degree distribution on non-negative integers $\{q_i\in(0,1) : \sum_{i=0}^\infty q_i = 1, 
\sum_{r=i}^\infty q_r>0,\forall i\geq 0\}$
{satisfying $\lim_{n \to \infty} \sum_{i=n+1}^\infty q_i \wpn{x} = 0$
for all $x\in[a,b]$,}
it holds 
\begin{align}
\Ec{ \mathbf{v}^\top \wpn{A} \mathbf{v}}{\mathbf{v},n} = \Sigma_f(A),\qquad 
\Ec{\nabla_\theta \mathbf{v}^\top \wpn{A} \mathbf{v}}{\mathbf{v},n} = \nabla_\theta \Sigma_f(A).\label{eq:unb2}
\end{align}
where the expectations are taken over the joint distribution on
random degree $n$ and Rademacher random vector $\mathbf v$ (other randomized probing vectors can be used as well).
\end{lemma}

The proof of Lemma~\ref{lmm:unbiased} is given in the supplementary material.
We emphasize that \eqref{eq:unb2} holds 
for any distribution $\{q_i\}^\infty_{i=0}$ on non-negative integers for which the
conditions stated in Lemma~\ref{lmm:unbiased} hold, e.g., geometric, Poisson or negative binomial distribution.

\subsection{Main result: optimal unbiased Chebyshev gradients} 
It is a well-known fact that stochastic gradient methods converge faster 
when the gradients have smaller variances. 
The variance of our proposed unbiased
estimators crucially depends on the choice of the degree distribution, i.e., $\{q_i\}^{\infty}_{i=0}$.
In this section, we design a degree distribution that is variance-optimal in some formal sense.
The variance of our proposed degree distribution decays exponentially with the expected degree, and this is crucial for for the convergence analysis (Section~\ref{sec:main2}).

The degrees-of-freedoms in choosing $\{q_i\}^\infty_{i=0}$
is infinite, which poses a challenge for devising low-variance distributions.
Our approach is based on the following simplified analytic approach studying
the scalar function $f$ in such a way that one can naturally expect that the resulting
distribution $\{q_i\}^\infty_{i=0}$ also provides low-variance for the matrix cases of \eqref{eq:unb2}.
We begin by defining the variance of randomized Chebyshev expansion \eqref{eq:estimator} via the Chebyshev weighted norm as
\begin{align} \label{eq:var}
\mathrm{Var}_C\left(\widehat{p}_n\right) := \mathbf{E}_n\left[\norm{\widehat{p}_n - f}_C^2\right],\quad
\mbox{where}~~\norm{g}_C^2 := \int_{-1}^1 \frac{g(\frac{b-a}2 x + \frac{b+a}2)^2}{\sqrt{1-x^2}} dx.
\end{align}
The primary reason why we consider the above variance is
because by utilizing
the orthogonality of Chebyshev polynomials we can derive an analytic expression for it.
\begin{lemma} \label{lmm:2}
Suppose $\{b_j\}_{j=0}^\infty$ are coefficients of the Chebyshev series for analytic function $f$
and $\widehat{p}_n$ is its randomized Chebyshev expansion \eqref{eq:estimator}. 
Then, it holds that
$\mathrm{Var}_C\left(\widehat{p}_n\right)
= \frac{\pi}2 \sum_{j=1}^\infty b_j^2 
\left( \frac{\sum_{i=0}^{j-1} q_i}{1 - \sum_{i=0}^{j-1} q_i} \right)$.
\end{lemma}

The proof of Lemma \ref{lmm:2} is given in the supplementary material.
One can observe from this result that the variance reduces as we assign larger 
masses to to high degrees (due to exponentially decaying property of $b_j$ \eqref{eq:decayrate}).
However, using large degrees increases the computational complexity of computing the estimators.
Hence, we aim to design a good distribution given some target complexity, i.e.,
the expected polynomial degree $N$. 
Namely, the minimization of $\mathrm{Var}_C\left(\widehat{p}_n\right)$ should be constrained by $\sum_{i=1}^\infty i q_i=N$ for some parameter $N\geq0$.

However, minimizing $\mathrm{Var}_C\left(\widehat{p}_n\right)$ subject to the aforementioned constraints
might be generally intractable as the number of variables $\{q_i\}^\infty_{i=0}$ is infinite
and the algebraic structure of $\{b_j\}^\infty_{i=0}$ is arbitrary. 
Hence, in order to derive 
an analytic or closed-form solution, we relax the optimization.
In particular, we suggest the following optimization to minimize an upper bound of the variance by utilizing $\abs{b_j}\leq 2U \rho^{-j}$ from \eqref{eq:decayrate} as follows:
\begin{align} \label{eq:degprob}
\min_{\{q_i\}_{i=0}^\infty} \ \sum_{j=1}^\infty \rho^{-2j} \left( \frac{\sum_{i=0}^{j-1} q_i}{1 - \sum_{i=0}^{j-1} q_i} \right)   
\quad 
\text{subject to} \ \ \  \sum_{i=1}^\infty i q_i = N, \sum_{i=0}^\infty q_i = 1 \ \ \text{and} \ \ q_i \geq 0. 
\end{align}
Figure \ref{fig:coeff} empirically demonstrates 
that $b_j^2 \approx c \rho^{-2j}$ for constant $c>0$ under $f(x)=\log x$,
in which case the above relaxed optimization \eqref{eq:degprob}
is nearly tight.
The next theorem establishes that \eqref{eq:degprob} has a closed-form solution,
despite having infinite degrees-of-freedom. The theorem is applicable when knowing a $\rho > 1$ and a bound $U$ such that the function $f$ is analytic with $\abs{f\left(\frac{b-a}2 z + \frac{b+a}2\right)} \leq U$ in the complex region bounded by the ellipse with foci $+1,-1$ and sum of major and minor semi-axis lengths is equal to $\rho > 1$. 
\begin{theorem} \label{thm:optdist}
Let $K = \max\{0, N - \left\lfloor\frac{\rho}{\rho-1}\right\rfloor\}$.
The optimal solution $\{q_i^*\}_{i=0}^\infty$ of \eqref{eq:degprob} is
\begin{align} \label{eq:optdist}
{
\textstyle
q_i^{*} = 
\begin{dcases}
0 &\text{for } \ i < K \\
1 - {(N-K)\left(\rho-1\right)}{\rho^{-1}} &\text{for } \ i = K \\
{(N-K)(\rho-1)^2}{\rho^{-i-1+N-K}} &\text{for } \ i > K,
\end{dcases}
}
\end{align}
and it satisfies the unbiasedness condition in Lemma \ref{lmm:unbiased}, i.e., 
$\lim_{n\to \infty} \sum_{i=n+1}^\infty q^{*}_i \wpn{x} = 0$.
\end{theorem}


The proof of Theorem \ref{thm:optdist} is given in the supplementary material.
Observe that a degree smaller than $K$ is never sampled under $\{ q_i^*\}$, which
means that
the corresponding unbiased estimator \eqref{eq:estimator} combines deterministic series of degree $K$ with randomized ones of higher degrees. 
Due to the geometric decay of  $\{q_i^*\}$ , large degrees will be sampled with exponentially small probability. 


The optimality of the proposed distribution \eqref{eq:optdist} (labeled opt) is illustrated by comparing 
it numerically to other distributions: negative binomial (labeled neg) and Poisson (labeled pois), on three
analytic functions: $\log x$, $\sqrt{x}$ and $\exp(x)$.
\cref{fig:log,fig:xsq,fig:exp} show the weighted variance \eqref{eq:var} of 
these distributions where their means are commonly set from $N=5$ to $100$.
Observe that the proposed distribution has order-of-magnitude smaller variance 
compared to other tested distributions.

\begin{figure*}[t]
\vspace{-0.15in}
\centering
\subfigure[\text{\scriptsize$f(x) = \log(x)$}]{\includegraphics[width=0.24\textwidth]{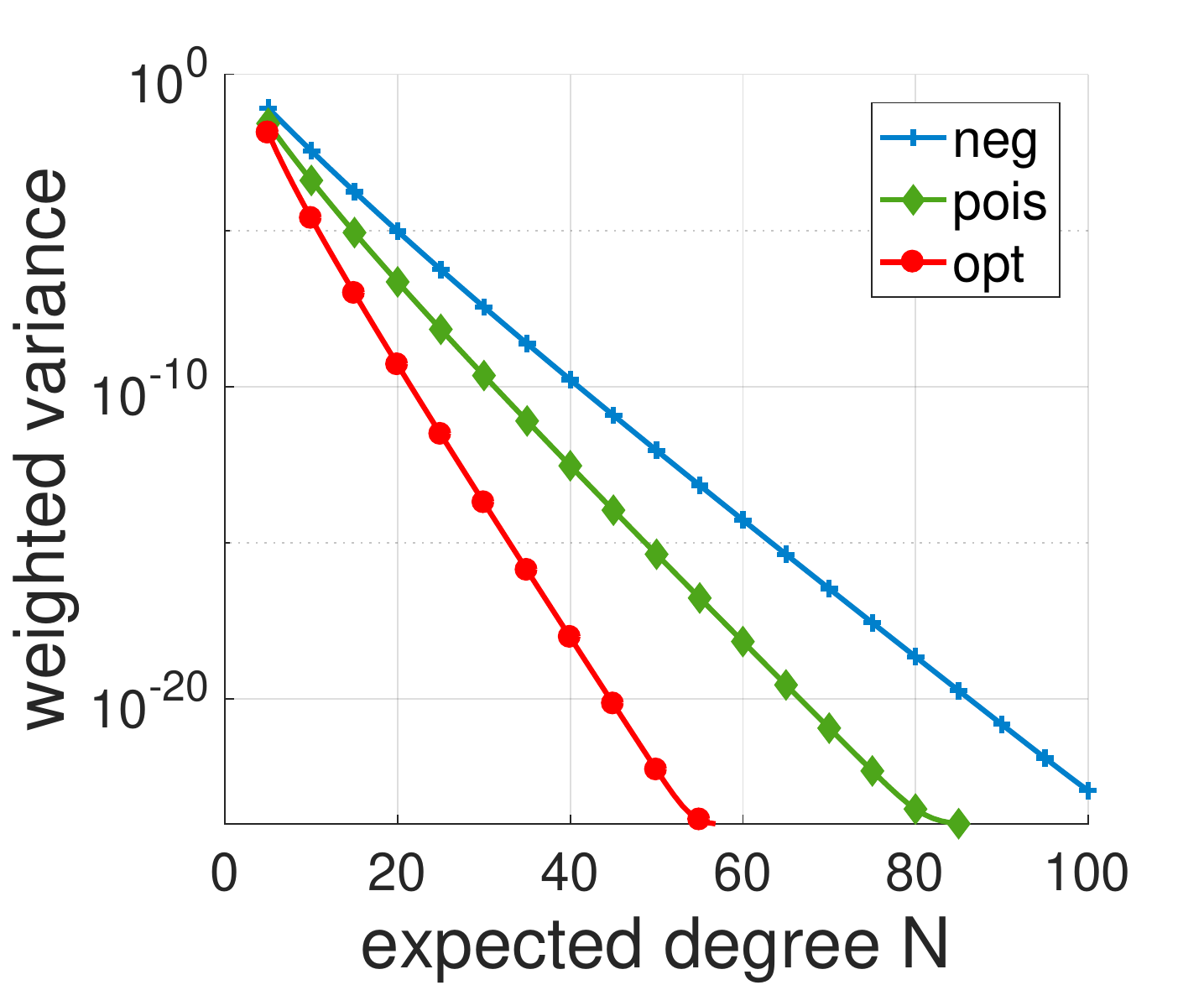}\label{fig:log}}
\subfigure[\text{\scriptsize$f(x) = x^{0.5}$}]{\includegraphics[width=0.24\textwidth]{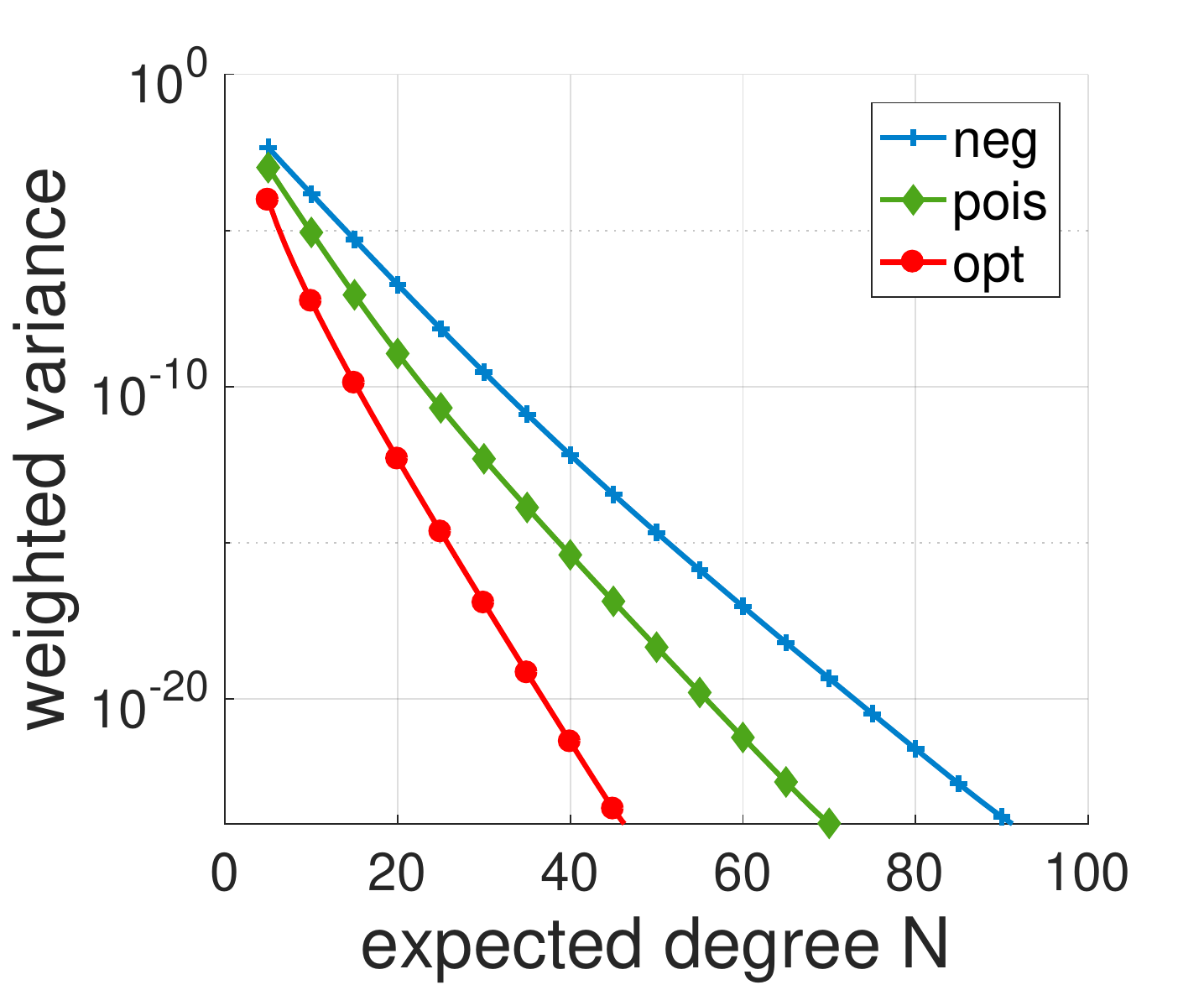}\label{fig:xsq}}
\subfigure[\text{\scriptsize$f(x) = \exp(x)$}]{\includegraphics[width=0.24\textwidth]{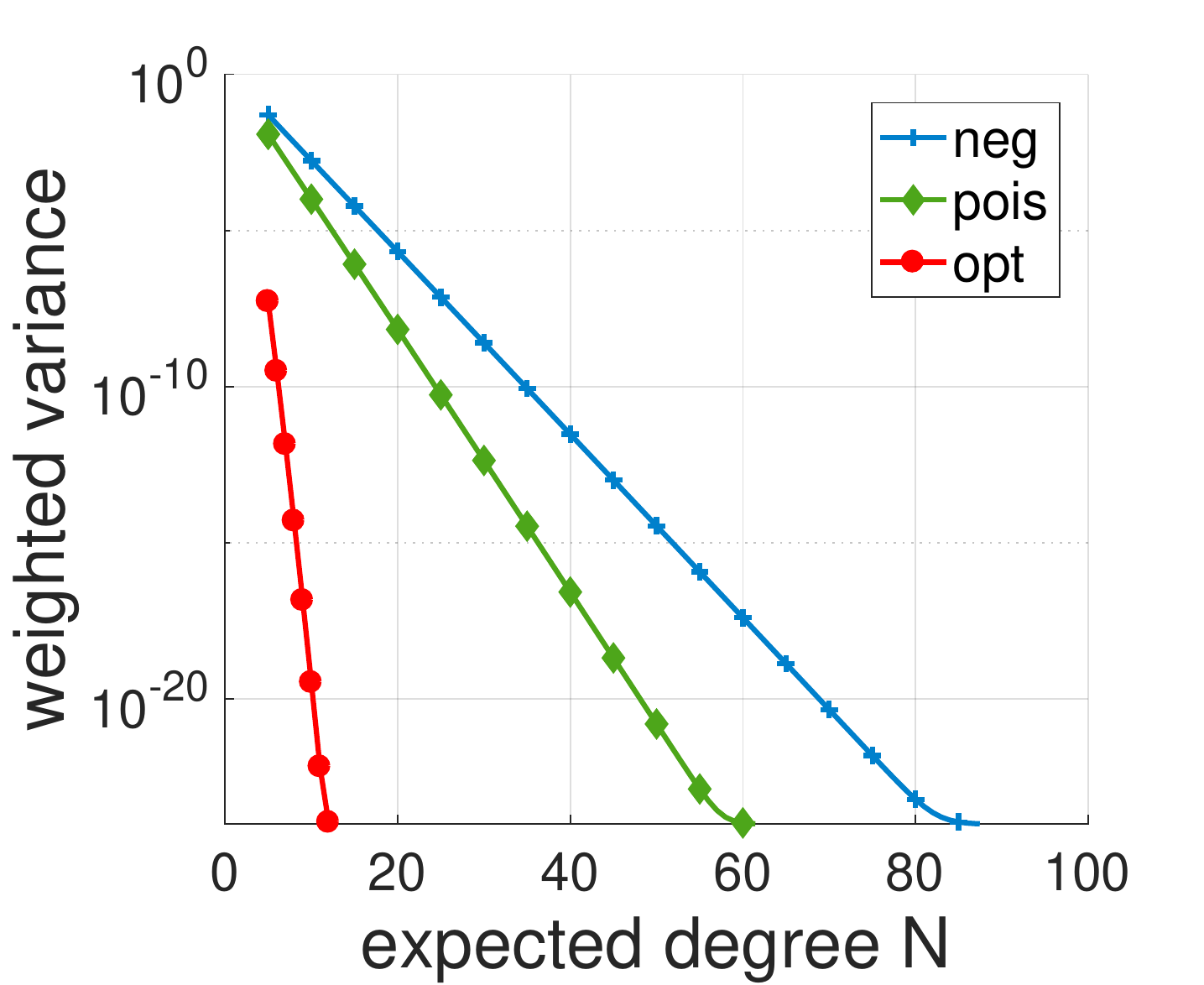}\label{fig:exp}}
\subfigure[\text{\scriptsize coefficients of $\log(x)$}]{\includegraphics[width=0.22\textwidth]{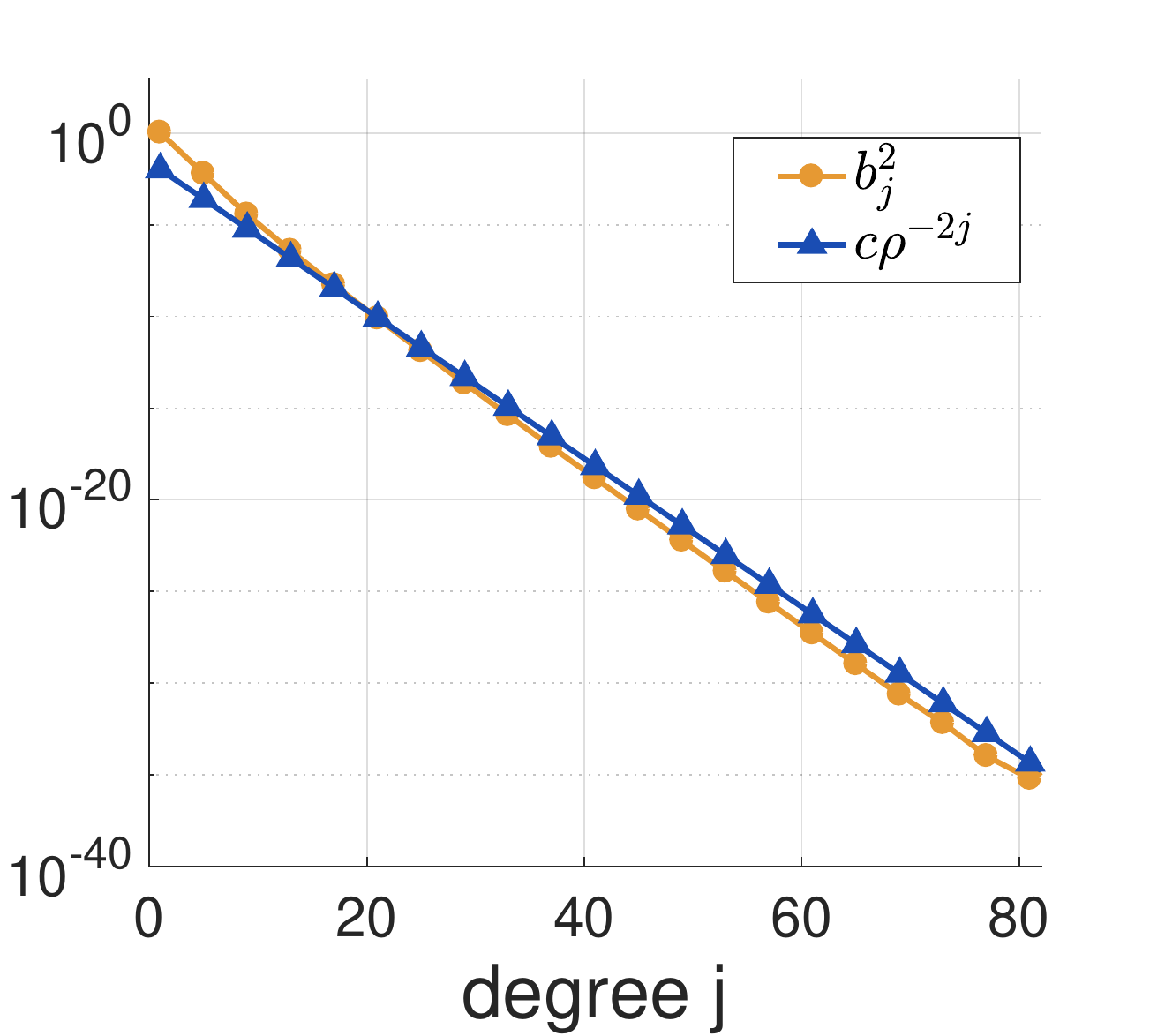}\label{fig:coeff}} 
\vspace{-0.1in}
\caption{Chebyshev weighted variance for three distinct distributions: negative binomial (neg), Poisson (pois) and the optimal distribution \eqref{eq:optdist} (opt) with the same mean $N$ under (a) $\log x$, (b) $\sqrt{x}$ on $[0.05,0.95]$ and (c) $\exp(x)$ on  $[-1,1]$, respectively. Observe that ``opt'' has the smallest variance among all distributions. (d) Comparison between $b_j^2$ and $c\rho^{-2j}$ for some constant $c>0$ and $\log x$.}\label{fig:var}
\vspace{-0.18in}
\end{figure*}

\section{{Stochastic Chebyshev gradient descent algorithms}} \label{sec:main2}
In this section, we leverage unbiased gradient estimators based on \eqref{eq:unb2} in conjunction with our 
optimal degree distribution \eqref{eq:optdist} to design computationally efficient methods for solving \eqref{eq:specopt}.
In particular, we propose to randomly sample a degree $n$ 
from \eqref{eq:optdist}
and estimate the gradient via Monte-Carlo method: 
\begin{align} \label{eq:estder2}
\frac{\partial}{\partial \theta_i} \Sigma_f(A)
= \E{ \frac{\partial}{\partial \theta_i} \vv^\top \wpn{A} \vv} 
\approx \frac{1}{M} \sum_{k=1}^M \vv^{(k)\top} \left( \sum_{j=0}^{n} \frac{b_j}{1 - \sum_{i=0}^{j-1} q^{*}_i} \frac{\partial \ww_j^{(k)}}{\partial \theta_i}\right)
\end{align}
where ${\partial \ww_j^{(k)}}/{\partial \theta_i}$ can be computed using a Rademacher vector $\vv^{(k)}$ and the recursive relation \eqref{eq:updatew}.

\subsection{Stochastic Gradient Descent (SGD)}
In this section, we consider the use of projected SGD in conjunction with \eqref{eq:estder2} to numerically solve the optimization \eqref{eq:specopt}. 
In the following, we provide a pseudo-code description of our proposed algorithm. 
\vspace{-0.05in}

\begin{algorithm}[H]
\caption{{SGD} for solving \eqref{eq:specopt}} \label{alg:sgd}
\begin{algorithmic}[1]
\STATE {\bf Input:} 
number of iterations $T$,
number of Rademacher vectors $M$, expected degree $N$ and $\theta^{(0)}$
\FOR{$t=0$ {\bf} to $T-1$}
\STATE  Draw $M$ Rademacher random vectors $\{\vv^{(k)}\}_{k=1}^{M}$
and a random degree $n$ from \eqref{eq:optdist} given $N$
\STATE Compute $\psi^{(t)}$ from \eqref{eq:estder2} at $\theta^{(t)}$ using $\{\vv^{(k)}\}_{k=1}^{M}$ and $n$
\STATE Obtain a proper step-size $\eta_t$
\STATE $\theta^{(t+1)} \leftarrow \proj{\theta^{(t)} - \eta_t \left(\psi^{(t)} + \nabla g(\theta^{(t)})\right)}$, 
where $\proj{\cdot}$ is the projection mapping into $\C$
\ENDFOR
\end{algorithmic}
\end{algorithm}
\vspace{-0.15in}

In order to analyze the convergence rate, 
we assume that  
$(\mathcal{A}0)$ all eigenvalues of $A(\theta)$ for $\theta \in \C'$ are in the interval $[a,b]$ for some open $\C'\supseteq\C$,
$(\mathcal{A}1)$ $\Sigma_f(A(\theta)) + g(\theta)$ is continuous and $\alpha$-strongly convex with respect to $\theta$ 
and $(\mathcal{A}2)$ $A(\theta)$ is $L_A$-Lipschitz for $\normf{\cdot}$,
$g(\theta)$ is $L_g$-Lipschitz and $\beta_g$-smooth.
The formal definitions of the assumptions are in the supplementary material.
These assumptions hold for many target applications, including the ones explored in Section~\ref{sec:exp}.
In particular, we note that assumption $(\mathcal{A}0)$ can be often satisfied with a careful choice of ${\cal C}$. 
It has been studied that (projected) SGD has a sublinear convergence rate for a smooth strongly-convex objective
if the variance of gradient estimates is uniformly bounded \cite{robbins1951stochastic,nemirovski2009robust}.
Motivated by this, we first
derive the following upper bound on the variance of gradient estimators 
under the optimal degree distribution \eqref{eq:optdist}.
\begin{lemma} \label{lmm:varbound}
Suppose that assumptions $(\mathcal{A}0)$-$(\mathcal{A}2)$ 
hold and $A(\theta)$ is $L_{\mathtt{nuc}}$-Lipschitz for $\normnuc{\cdot}$.
Let $\psi$ be the gradient estimator \eqref{eq:estder2} at $\theta \in \C$ 
using Rademacher vectors $\{\vv^{(k)}\}_{k=1}^M$ and degree $n$ drawn from the optimal distribution \eqref{eq:optdist}. Then, 
$
\mathbf{E}_{\vv,n} [\norm{\psi}_2^2] \leq
\left({2 L_A^2}/{M} + d^\prime L_{\mathtt{nuc}}^2\right)
\left(C_1 + {C_2 N^4}{\rho^{-2N}}\right)
$
where $C_1,C_2>0$ are some constants independent of $M,N$.
\end{lemma}
The above lemma allows us to provide a sublinear convergence rate for Algorithm \ref{alg:sgd}.

\begin{theorem} \label{thm:sgd}
Suppose that assumptions $(\mathcal{A}0)$-$(\mathcal{A}2)$ hold 
and $A(\theta)$ is $L_{\mathtt{nuc}}$-Lipschitz for $\normnuc{\cdot}$.
If one chooses the step-size $\eta_t = {1}/{\alpha t}$, then it holds that
\begin{align*}
\mathbf{E}[\norm{\theta^{(T)} - \theta^*}_2^2] \leq \frac{4}{\alpha^2 T} \max \left(L_g^2, 
\left(\frac{2 L_A^2}{M} + d^\prime L_{\mathtt{nuc}}^2\right)
\left(C_1 + \frac{C_2 N^4}{\rho^{2N}}\right)
\right)
\end{align*}
where $C_1, C_2 > 0$ are constants independent of $M,N$, and
$\theta^* \in \C$ is the global optimum of \eqref{eq:specopt}.
\end{theorem}
The proofs of Lemma \ref{lmm:varbound} and Theorem \ref{thm:sgd} are given in the supplementary material.
Note that {larger $M,N$ provide better convergence but they increase the computational complexity.}
The convergence is also faster with smaller $d^\prime$, which is also evident
in our experiments (see Section \ref{sec:exp}).

\vspace{-0.1in}
\subsection{Stochastic Variance Reduced Gradient (SVRG)}
\vspace{-0.05in}
In this section, we introduce a more advanced stochastic method using
a further variance reduction technique, 
inspired by the stochastic variance reduced gradient method (SVRG) \cite{johnson2013accelerating}.
The full description of the proposed SVRG scheme for solving the optimization \eqref{eq:specopt}
is given below.

\begin{algorithm}[H]
\caption{{SVRG} for solving \eqref{eq:specopt}} \label{alg:svrg}
{
\begin{algorithmic}[1]
\STATE {\bf Input:}  
number of inner/outer iterations $T,S$,
number of Rademacher vectors $M$, 
expected degree $N$, step-size $\eta$ and initial parameter $\theta^{(0)} \in \C$
\STATE $\widetilde{\theta}^{(1)} \leftarrow \theta^{(0)}$
\FOR {$s = 1$ to $S$}
\STATE $\widetilde{\mu}^{{(s)}} \leftarrow \nabla \Sigma_f(A(\widetilde{\theta}^{(s)}))$ and $\theta^{(0)} \leftarrow \widetilde{\theta}^{(s)}$
\FOR {$t = 0$ to $T-1$}
\STATE Draw $M$ Rademacher random vectors $\{\vv^{(k)}\}_{k=1}^M$ and a random degree $n$ from \eqref{eq:optdist}
\STATE Compute $\psi^{(t)},\widetilde{\psi}^{(s)}$ from \eqref{eq:estder2} at $\theta^{(t)}$ and $\widetilde{\theta}^{(s)}$, respectively using $\{\vv^{(k)}\}_{k=1}^M$ and $n$
\STATE $\theta^{(t+1)} \leftarrow \proj{\theta^{(t)} - \eta \left( 
\psi^{(t)} - \widetilde{\psi}^{(s)}
+ \widetilde{\mu}^{(s)} + \nabla g(\theta^{(t)})\right)}$
\ENDFOR
\STATE $\widetilde{\theta}^{(s+1)} \leftarrow \frac1T\sum_{t=1}^T \theta^{(t)}$
\ENDFOR
\end{algorithmic}
}
\end{algorithm}
The main idea of SVRG is 
to subtract a mean-zero random variable to the original stochastic gradient estimator,
where the randomness between them is shared. 
The SVRG algorithm 
was originally designed for optimizing finite-sum objectives, i.e., $\sum_{i} f_i(x)$, whose
randomness is from the index $i$. 
On the other hand, 
the randomness in our case is from polynomial degrees and trace probing vectors
for optimizing objectives of spectral-sums.
This leads us to use the same randomness in $\{\vv^{(k)}\}_{k=1}^{M}$ and $n$ for estimating both $\psi^{(t)}$ and $\widetilde{\psi}^{(s)}$ in
line 7 of Algorithm \ref{alg:svrg}. 
We remark that unlike SGD,
Algorithm \ref{alg:svrg} requires the expensive computation of exact gradients every $T$ iterations.
The next theorem establishes that if one sets $T$ correctly only $O(1)$ gradient computations are required (for a fixed suboptimality)
since we have a linear convergence rate. 
\begin{theorem} \label{thm:svrg}
Suppose that assumptions $(\mathcal{A}0)$-$(\mathcal{A}2)$ hold 
and $A(\theta)$ is $\beta_A$-smooth for $\normf{\cdot}$.
Let $\beta^2 = 2\beta_g^2 + \left(\frac{L_A^4 + \beta_A^2}{M} + L_A^4 \right)\left(D_1 + \frac{D_2 N^8}{\rho^{2N}} \right)$
for some constants $D_1, D_2 > 0$ independent of $M,N$.
Choose $\eta = \frac{\alpha}{7 \beta^2}$ and $T \geq 25 \beta^2 / \alpha^2$. Then, it holds that 
\begin{align*}
\mathbf{E}[\norm{\widetilde{\theta}^{(S)} - \theta^*}_2^2] \leq r^S
\mathbf{E}[\norm{\theta^{(0)} - \theta^*}_2^2],
\end{align*}
where 
$0 < r < 1$ is some constant 
and $\theta^* \in \C$ is the global optimum of \eqref{eq:specopt}.
\end{theorem}
The proof of the above theorem is given in the supplementary material, where
we utilize the recent analysis of SVRG for the sum of smooth non-convex objectives \cite{garber2015fast, allen2016improved}. 
The key additional component in our analysis
is to characterize $\beta>0$ in terms of $M, N$
so that the unbiased gradient estimator \eqref{eq:estder2} is $\beta$-smooth in expectation under the optimal degree distribution \eqref{eq:optdist}. 


\section{Applications} \label{sec:exp}

In this section, 
we apply the proposed methods to two machine learning tasks: matrix completion and
learning Gaussian processes. These
correspond to minimizing spectral-sums 
$\Sigma_f$ with $f(x)=x^{1/2}$ and $\log x$, respectively.
We evaluate our methods under real-world datasets for both experiments.

\subsection{Matrix completion} 
The goal 
is to recover a low-rank matrix $\theta \in [0,5]^{d \times r}$ 
when a few of its entries 
are given.
Since the rank function is neither differentiable nor convex, 
its relaxation such as Schatten-$p$ norm has been used in respective optimization formulations. 
In particular, we consider the smoothed nuclear norm (i.e., Schatten-$1$ norm) minimization \cite{lu2015smoothed,mohan2012iterative}
that 
corresponds to
\begin{align*}
\min_{\theta \in [0,5]^{d \times r}} 
\mathtt{tr} ( A^{1/2}) + 
\lambda \sum_{(i,j)\in \Omega} \left( \theta_{i,j} - R_{i,j}\right)^2
\end{align*}
where $A = \theta \theta^\top + \varepsilon I$,
$R \in [0,5]^{d \times r}$ is a given matrix with missing entries, 
$\Omega$ indicates the positions of known entries and
$\lambda$ is a weight parameter and $\varepsilon>0$ is a smoothing parameter.
Observe that $\norm{A}_{\mathtt{mv}} = \norm{\theta}_{\mathtt{mv}} = O(d r)$,
and the derivative estimation in this case can be amortized to compute 
using $O(d M (N^2 + Nr))$ operations.
More details on this and our experimental settings are given in the supplementary material.

\begin{figure}[t]
\centering
\subfigure[]{\includegraphics[width=0.24\textwidth]{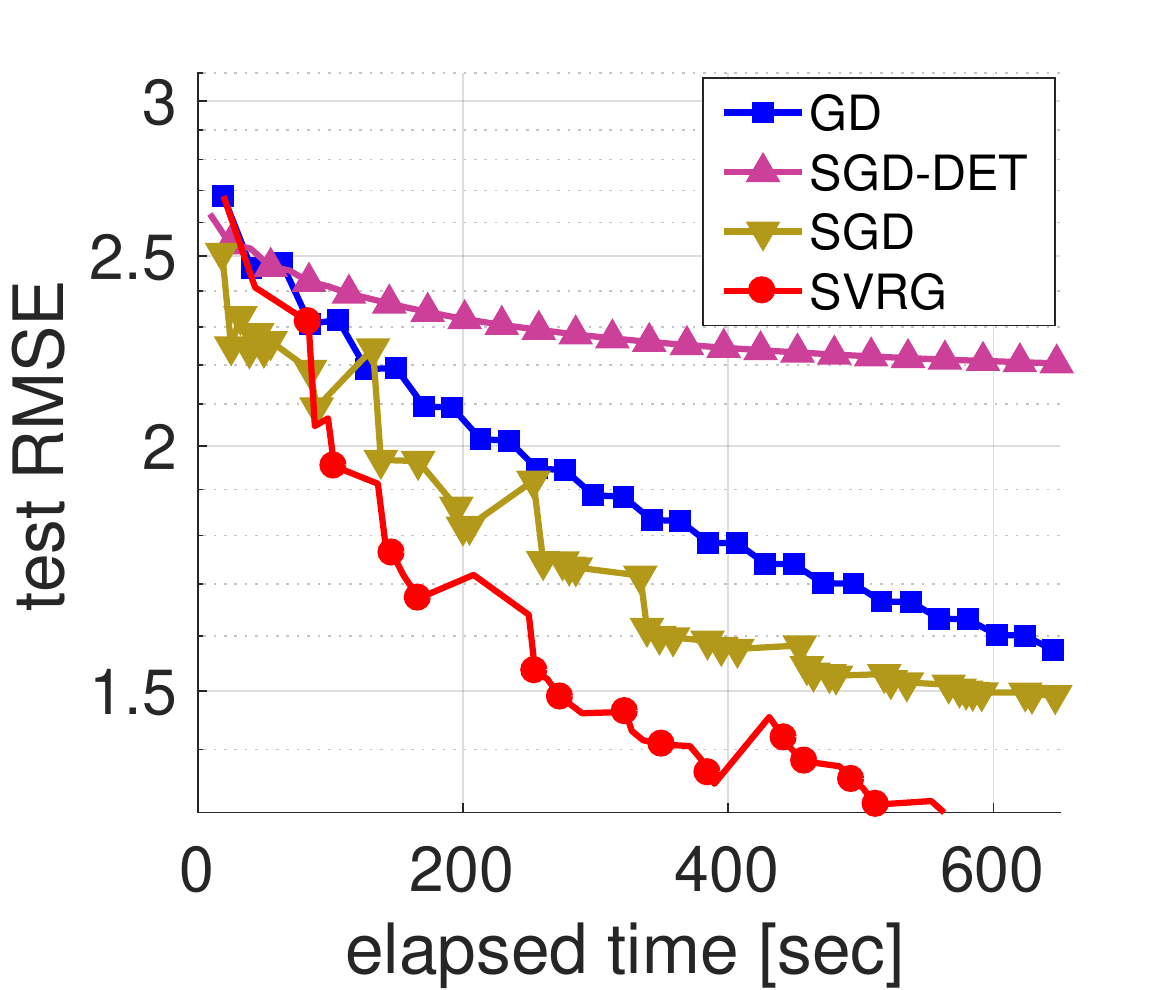} \label{fig:mc1}} \hspace{-0.1in}
\subfigure[]{\includegraphics[width=0.24\textwidth]{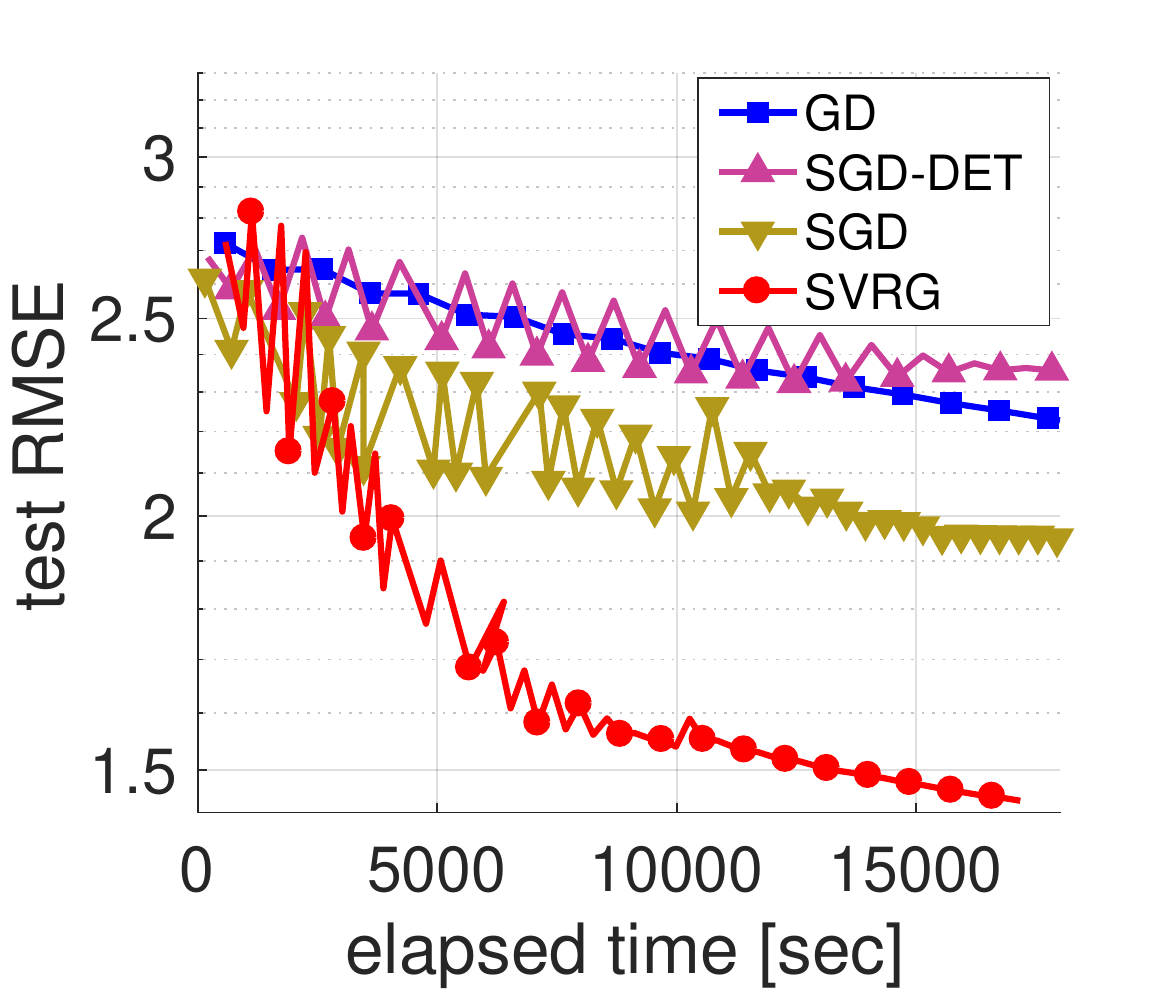}\label{fig:mc2}} \hspace{-0.1in}
\subfigure[]{\includegraphics[width=0.24\textwidth]{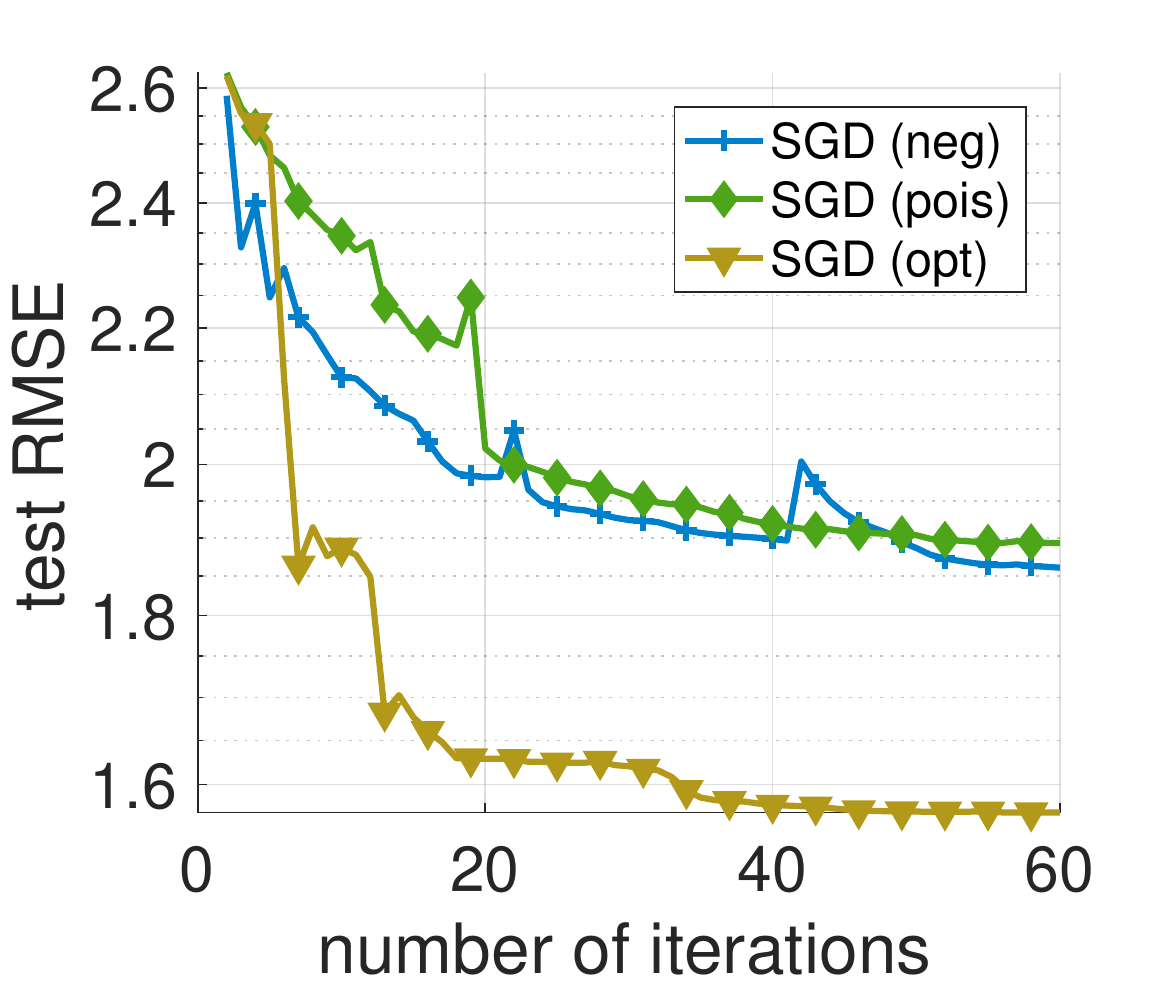}\label{fig:mc3}}
\subfigure[]{\includegraphics[width=0.24\textwidth]{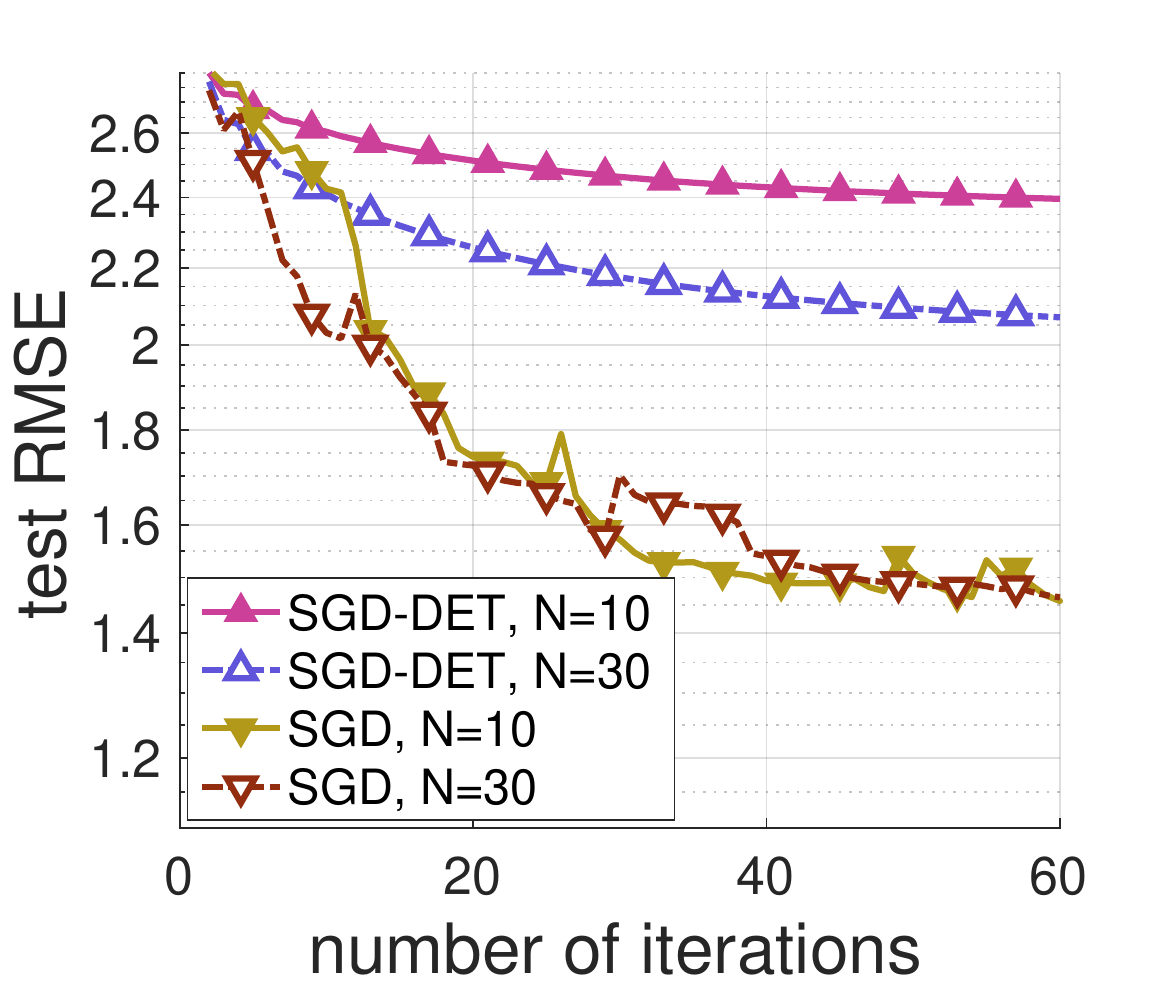}\label{fig:mc4}}
\vspace{-0.1in}
\caption{Matrix completion results 
under (a) MovieLens 1M and (b) MovieLens 10M. 
(c) Algorithm \ref{alg:sgd} (\textsc{SGD}) in MovieLens 1M under 
other distributions such as negative binomial (neg) and Poisson (pois).
(d) \textsc{SGD} and \textsc{SGD-DET} 
under $N=10, 30$.
}
\vspace{-0.12in}
\end{figure}

We use the MovieLens 1M and 10M datasets \cite{movielens} (they correspond to $d=3,706$ and $10,677$, respectively)
and benchmark
the gradient descent ({GD}), 
Algorithm \ref{alg:sgd} ({SGD})
and Algorithm \ref{alg:svrg} ({SVRG}).
We also consider
a variant of {SGD} using a deterministic polynomial degree, referred as {SGD-DET}, 
where it uses biased gradient estimators.
We report the results for MovieLens 1M in Figure \ref{fig:mc1} and 10M in \ref{fig:mc2}.
For both datasets, {SGD-DET} performs badly due to
its biased gradient estimators.
On the other hand, {SGD} 
converges much faster and outperforms {GD},
where {SGD} for 10M converges much slower than that for 1M 
due to the larger dimension 
$d^\prime = d r$ 
(see Theorem \ref{thm:sgd}).
Observe that {SVRG} is the fastest one,
e.g., compared to {GD},
about 2 times faster to achieve RMSE $1.5$ for MovieLens 1M and up to 6 times faster to achieve RMSE $1.8$ for MovieLens 10M as shown in Figure \ref{fig:mc2}.
The gap between {SVRG} and {GD} is expected to increase for larger datasets.
We also test {SGD} under other degree distributions: negative binomial (neg) and Poisson (pois) by choosing parameters so that their means equal to $N=15$. 
As reported in Figure \ref{fig:mc3},
other distributions have relatively large variances so that 
they converge slower than the optimal distribution (opt).
{In Figure \ref{fig:mc4}, we compare {SGD-DET} with {SGD} of the optimal distribution under the (mean) polynomial degrees $N=10,30$. 
Observe that a larger degree ($N=30$) reduces the bias error in {SGD-DET},
while {SGD} achieves similar error regardless of the degree.
The above results confirm that 
the unbiased gradient estimation and our degree distribution \eqref{eq:optdist}  
are crucial for SGD.
}

\subsection{\bf Learning for Gaussian process regression}
Next, we apply our method to hyperparameter learning for Gaussian process (GP) regression.
Given training data $\left\{\mathbf{x}_i \in \mathbb{R}^{\ell}\right\}_{i=1}^{d}$ with corresponding outputs $\mathbf{y}\in \mathbb{R}^{d}$,
the goal of GP regression is to learn a hyperparameter $\theta$ for predicting the output of a new/test input.
The hyperparameter $\theta$ constructs the kernel matrix $A(\theta) \in \SM$ of the training data $\{\mathbf{x}_i\}_{i=1}^d$ (see \cite{rasmussen2004gaussian}).
One can find a good hyperparameter by minimizing the negative log-marginal likelihood with respect to $\theta$: 
\begin{align*}
\mathcal{L}
&:= -\log p\left(\mathbf{y} | \{ \mathbf{x}_i\}_{i=1}^d \right) 
= \frac12 \mathbf{y}^\top A(\theta)^{-1}\mathbf{y} +\frac12 \log \det A(\theta) + \frac{n}{2}\log 2\pi.
\end{align*}
For handling large-scale datasets, \cite{wilson2015kernel} proposed the structured kernel interpolation framework 
assuming $\theta = [\theta_i] \in \mathbb{R}^3$ and
\begin{align*}
A(\theta) = W K W^\top + \theta_1^2 I, \quad
K_{i,j} = \theta_2^2 \exp \left( {\norm{\mathbf{x}_i - \mathbf{x}_j}_2^2}/{2 \theta_3^2}\right),
\end{align*}
where $W \in \mathbb{R}^{d \times r}$ is some sparse matrix
and $K \in \mathbb{R}^{r \times r}$ is a dense kernel with $r \ll d$.
Specifically, in \cite{wilson2015kernel}, $r$ ``inducing'' points are selected 
and entries of $W$ are computed via interpolation with the inducing points.
Under the framework, matrix-vector multiplications with $A$ can be performed even faster, 
requiring $\norm{A}_{\mathtt{mv}} = \norm{W}_{\mathtt{mv}} + \norm{K}_{\mathtt{mv}} = O(d + r^2)$ operations.
From $\|A\|_{\mathtt{mv}} = \| \frac{\partial A}{\partial \theta_i}\|_{\mathtt{mv}}$ and $d^\prime = 3$, 
the complexity for computing gradient estimation \eqref{eq:estder2} becomes $O(MN (d + r^2))$. 
If we choose $M,N,r=O(1)$, the complexity reduces to $O(d)$.
The more detailed problem description and our experimental settings are given in the supplementary material.

\begin{wrapfigure}[15]{r}{0.45\textwidth}
\vspace{-0.2in}
\begin{center}
\hspace{-0.1in}
\subfigure[]{\includegraphics[width=0.24\textwidth]{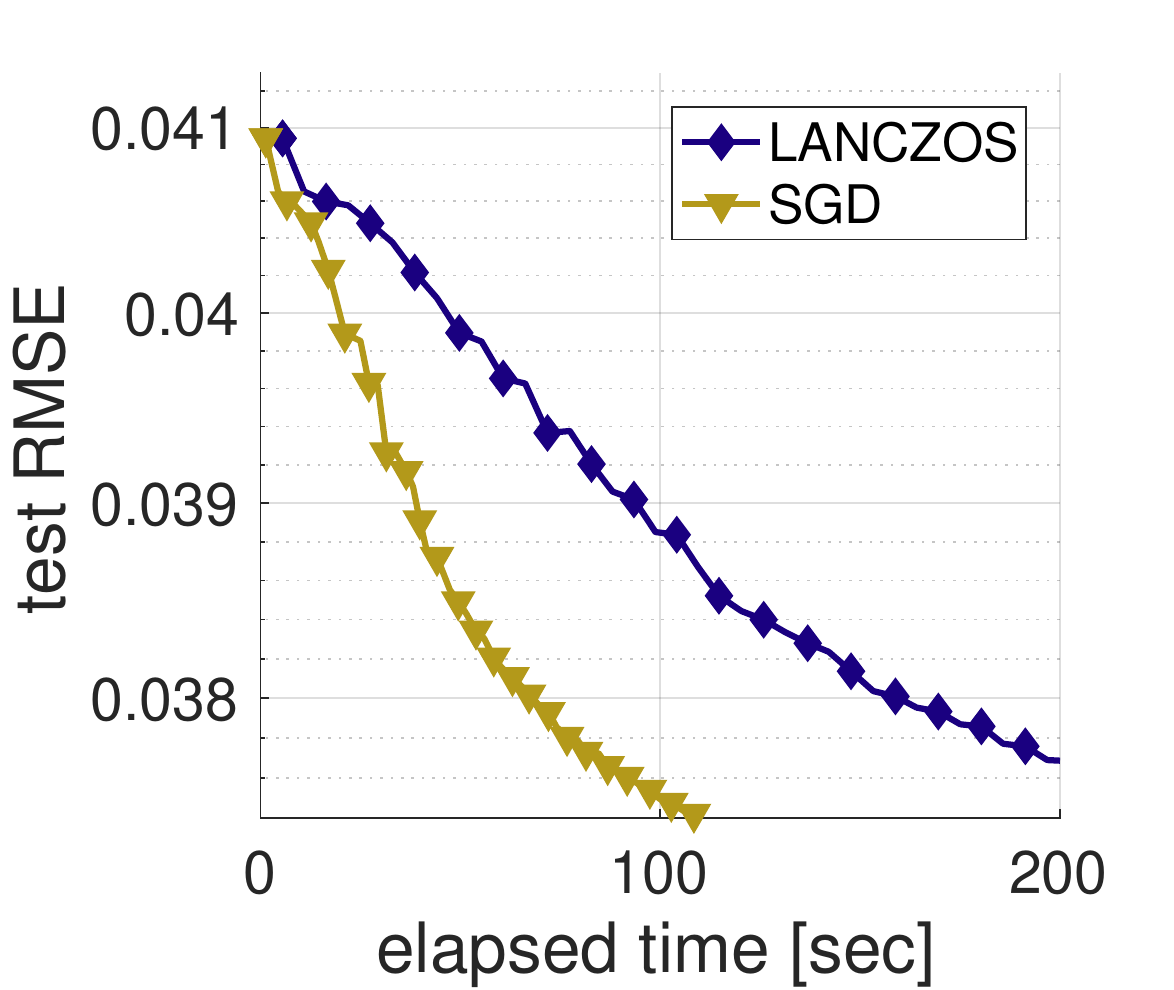}} \hspace{-0.18in}
\subfigure[]{\includegraphics[width=0.24\textwidth]{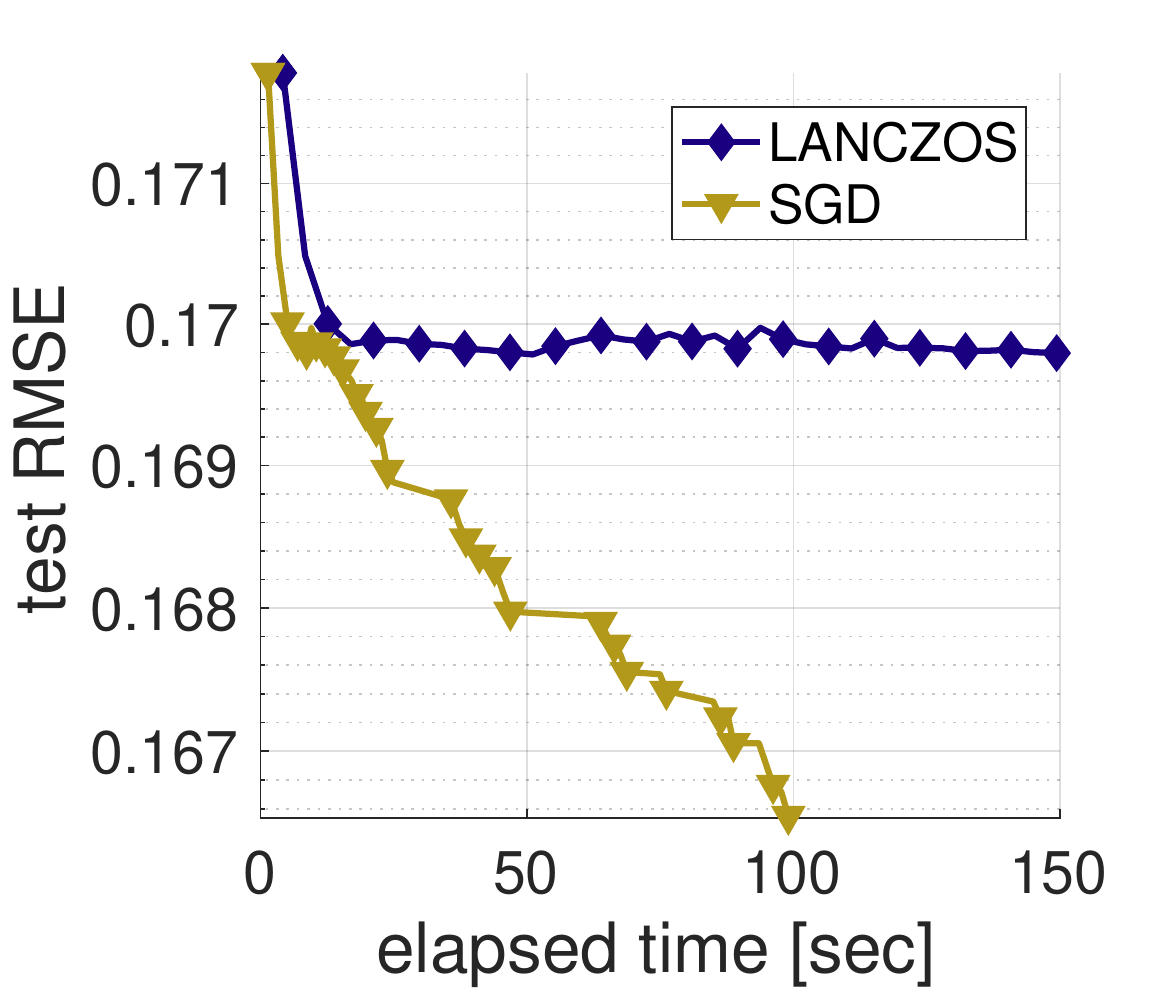}}
\end{center}
\vspace{-0.1in}
\caption{Hyperparameter learning for Gaussian process in modeling (a) sound dataset and (b) Szeged humid dataset 
comparing \textsc{SGD} to stochastic Lanczos quadrature (\textsc{LANCZOS}).
}\label{fig:gp}
\end{wrapfigure}
We benchmark GP regression under natural sound dataset used in \cite{wilson2015kernel} and Szeged humid dataset \cite{szeged} 
where they correspond to $d=35,000$ and $16,930$, respectively.
Recently, \cite{dong2017scalable} utilized 
an approximation to derivatives of log-determinant
based on
stochastic Lanczos quadrature \cite{ubaru2017fast} (LANCZOS). 
We compare it with Algorithm \ref{alg:sgd} ({SGD})
which utilizes with unbiased gradient estimators
while SVRG requires the exact gradient computation at least once
which is intractable to run in these cases. 
As reported in Figure \ref{fig:gp},
{SGD} converges faster than {LANCZOS} for both datasets 
and it runs $2$ times faster to achieve RMSE $0.0375$ under sound dataset 
and under humid dataset {LANCZOS} can be often stuck at a local optimum, 
while SGD avoids it due to the use of unbiased gradient estimators.

\section{Conclusion}
We proposed an optimal variance unbiased estimator for spectral-sums and their gradients.
We applied our estimator in the SGD and SVRG frameworks, and analyzed convergence.
The proposed optimal degree distribution is a crucial component of the analysis. 
We believe that the proposed stochastic methods are of broader interest in many machine learning tasks
involving spectral-sums.

\subsection*{Acknowledgement}
This work was supported by the National Research Foundation of Korea(NRF) grant funded by the Korea government(MSIT) (2018R1A5A1059921). Haim Avron acknowledges the support of the Israel Science Foundation (grant no. 1272/17).

\bibliography{biblist}
\bibliographystyle{icml2018}

\clearpage
\title{
Stochastic Chebyshev Gradient Descent \\
for Spectral Optimization
}

\appendix

\section{Details of experiments}
\subsection{Matrix completion}
For matrix completion, 
the problem can be expressed via the convex smoothed nuclear norm minimization as 
\begin{align} \label{eq:mc}
\min_{\theta \in [0,5]^{d \times r}} 
\mathtt{tr} ( A^{1/2}) + 
\lambda \sum_{(i,j)\in \Omega} \left( \theta_{i,j} - R_{i,j}\right)^2,
\end{align}
where $A = \theta \theta^\top + \varepsilon I$,
$R \in [0,5]^{d \times r}$ is a given matrix with missing entries, 
$\Omega$ indicates the positions of known entries and
$\lambda$ is a weight parameter and $\varepsilon>0$ is a smoothing parameter. 
In this case, the gradient estimator $\eqref{eq:estder2}$ can be amortized as
\begin{align} \label{eq:amort}
\nabla_\theta \mathtt{tr}(A^{p/2}) 
&\approx
\frac{2}{M} \sum_{k=1}^M \sum_{i=0}^{n-1}
\left( 2 - \mathds{1}_{i=0}\right) \ww_{i}^{(k)}
\left( \sum_{j=i}^{n-1} \frac{b_{j+1}}{1-\sum_{\ell=0}^{j}q^{*}_\ell} \yy_{j-i}^{(k)} \right)^\top\theta
\end{align}
where 
\begin{align*}
\ww_{j+1}^{(k)} &= 2\ww_{j}^{(k)} - \ww_{j-1}^{(k)},\quad \ww_1^{(k)} = \widetilde{A} \vv, \quad \ww_0^{(k)} = \vv^{(k)},\\
\yy_{j+1}^{(k)} &= 2 \ww_{j+1}^{(k)} + \yy_{j-1}^{(k)},\quad \yy_1^{(k)} = 2\widetilde{A}\vv^{(k)}, \quad \yy_0^{(k)}=\vv^{(k)}
\end{align*}
and $\widetilde{A} = \left(\frac{2}{b-a}A -\frac{b+a}{b-a}I \right)$ 
for the lower/upper bound on $A$'s eigenvalues $a,b \in \mathbb{R}^+$.
This comes from the following lemma, whose proof is in Section \ref{sec:lmms}.
\begin{lemma} \label{lmm:amort}
Suppose $f$ is an analytic function and $p_n(x):= \sum_{j=0}^n b_j T_j(x)$ is its truncated Chebyshev series of degree $n\geq 1$ for $x \in [-1,1]$. Let $A = \theta \theta^\top + \varepsilon I $ for $\theta \in \mathbb{R}^{d \times r}, \varepsilon > 0$ such that all eigenvalues of $A$ are in $[-1,1]$. Then, for any $\vv \in \mathbb{R}^d$, it holds that
\begin{align*}
\nabla_\theta 
\vv^\top p_n(A) \vv 
= 2 \sum_{i=0}^{n-1} \left(2 - \mathds{1}_{i=0}\right) \ww_i \Bigg(\sum_{j=i}^{n-1} b_{j+1} \yy_{j-i}\Bigg)^\top \theta,
\end{align*}
where $\ww_{j+1} = 2A\ww_j - \ww_{j-1}, \ww_1 = A \vv, \ww_0 = \vv$  and  $\yy_{j+1} = 2\ww_{j+1} + \yy_{j-1}, \yy_1 = 2A \vv, \yy_0 = \vv$.
\end{lemma}

Observe that $\norm{A}_{\mathtt{mv}} = \norm{\theta}_{\mathtt{mv}} = O(d r)$,
and the computation for \eqref{eq:amort} can be amortized 
using $O(M (n^2 d + n d r))$ operations. For $M,n,r = O(1)$, the complexity reduces to $O(d)$.

After update the parameter $\theta$ in a direction of gradient estimator, 
we project $\theta$ onto $[0,5]^{d \times r}$, that is,
\begin{align*}
\proj{\theta_{i,j}} = 
\begin{dcases}
\theta_{i,j}, \quad &\text{ if  } \theta_{i,j} \in [0,5], \\
0, \quad &\text{ if  } \theta_{i,j} < 0, \\
5, \quad &\text{ otherwise.}
\end{dcases}
\end{align*}
In addition, after performing all gradient updates, we finally apply low-rank approximation using truncated SVD with rank $10$ once and measure the test root mean square error (RMSE). 

\textbf{Setup.}
We use matrix $R$ from MovieLens 1M (about $10^6$ integer ratings from $1$ to $5$ from $6,040$ users on $3,706$ movies)
and 10M (about $10^7$ ratings from $0.5$ to $5$ with intervals $0.5$ from $10,677$ users on $71,567$ movies) 
datasets \cite{movielens}.
We randomly select $90\%$ of each dataset for training and use the rest for testing.
We choose the (mean) polynomial degree $N=15$ and the number of trace random vectors $M=100$ for {SVRG} 
and $M = 200$ for {SGD-DET, SGD}, respectively, 
for comparable complexity at each gradient update.
Especially, for \textsc{SVRG}, we choose $T=100$.
We decrease step-sizes exponentially with ratio $0.97$ over the iterations.

\subsection{Gaussian process (GP) regression}
Given training data $\left\{\mathbf{x}_i\in \mathbb{R}^\ell\right\}_{i=1}^{d}$ with corresponding outputs $\mathbf{y}\in \mathbb{R}^{d}$,
the goal of GP regression is to learn a hyperparameter $\theta$ for predicting the output of a new/test input.
GP defines a distribution over functions,
which follow multivariate Gaussian distribution with
mean function $\mu_\theta:\mathbb{R}^{\ell}\rightarrow \mathbb{R}$ 
and covariance (i.e., kernel) function $a_\theta: \mathbb{R}^{\ell} \times \mathbb{R}^{\ell} \rightarrow \mathbb{R}$.
To this end, we set the kernel matrix $A=A(\theta) \in \mathcal S^{d \times d}$ of $\{\mathbf{x}_i\}_{i=1}^{d}$
such that $A_{i,j} = a_\theta\left(\mathbf{x}_i,\mathbf{x}_j\right)$ and the mean function to be zero. 
One can find a good hyperparameter by minimizing the negative log-marginal likelihood with respect to $\theta$: 
\begin{align} \label{eq:gplik}
\mathcal{L}
&:= -\log p\left(\mathbf{y} | \{ \mathbf{x}_i\}_{i=1}^d \right) 
= \frac12 \mathbf{y}^\top A^{-1}\mathbf{y} +\frac12 \log \det A + \frac{n}{2}\log 2\pi,
\end{align}
and predict ${y} = \mathbf{a}^\top A^{-1} \mathbf{y}$
where $\mathbf{a}_{i} = a_{\theta}(\mathbf{x}_i, \mathbf{x})$  (see \cite{rasmussen2004gaussian}).
Gradient-based methods can be used for optimizing \eqref{eq:gplik} using its partial derivatives:
\begin{align*}
\frac{\partial \mathcal{L}}{\partial \theta_i}
= -\frac12 \left( \mathbf{y}^\top \frac{\partial A}{\partial \theta_i}\right) A^{-1} \left( \frac{\partial A}{\partial \theta_i} \mathbf{y}\right)-\frac12 
\frac{\partial\log \det A}{\partial \theta_i}.
\end{align*}
Observe that the first term can be computed by an efficient linear solver, e.g., conjugate gradient descents \cite{saad2003iterative},
while the second term is computationally expensive for large $d$.
Hence, one can use our proposed gradient estimator \eqref{eq:estder2} for $\Sigma_f(A)$ with $f(x)=\log x$.

For handling large-scale datasets, \cite{wilson2015kernel} proposed the structured kernel interpolation framework 
assuming $\theta = [\theta_i] \in \mathbb{R}^3$ and
\begin{align*}
A(\theta) = W K W^\top + \theta_1^2 I, 
K_{i,j} = \theta_2^2 \exp \left( {\norm{\mathbf{x}_i - \mathbf{x}_j}_2^2}/{2 \theta_3^2}\right),
\end{align*}
where $W \in \mathbb{R}^{d \times r}$ is some sparse matrix
and $K \in \mathbb{R}^{r \times r}$ is a dense kernel with $r \ll d$.
Specifically, the authors select $r$ ``inducing'' points
and compute entries of $W$ via interpolation with the inducing points.
Under the framework, matrix-vector multiplications with $A$ can be performed even faster, 
requiring $\norm{A}_{\mathtt{mv}} = \norm{W}_{\mathtt{mv}} + \norm{K}_{\mathtt{mv}} = O(d + r^2)$ operations.
From 
$\|A\|_{\mathtt{mv}} = \| \frac{\partial A}{\partial \theta_i}\|_{\mathtt{mv}}$ and $d^\prime = 3$, 
the complexity for computing gradient estimation \eqref{eq:estder2}
becomes $O(MN (d + r^2))$. 
If we choose $M,N,r=O(1)$, the complexity reduces to $O(d)$.

\textbf{Setup.} We benchmark GP regression under natural sound dataset used in \cite{wilson2015kernel, dong2017scalable} and 
Szeged humid data \cite{szeged}.
We randomly choose $35,000$ points for training and $691$ for testing in sound dataset
and choose $16,930$ points for training and $614$ points for test in Szeged 2015-2016 humid dataset.
We set the polynomial degree $N=15$ and $M=30$ trace vectors for all algorithms.
We also select $r = 3000$ induced points for kernel interpolation.
Since GP regression is non-convex problem, the gradient descent methods are sensitive to the initial point.
We select a good initial point using random grid search. 
We observe that our algorithm (SGD) utilizing unbiased gradient estimator performs well for any initial point.
On the other hand, since LANCZOS is type of biased gradient descent methods, 
it is often stuck on a bad local optimum.

\section{Proof of theorems}
\subsection{Smoothness and strong convexity of matrix functions}
We first provide the formal definitions of the assumptions in Section \ref{sec:main2}.
Let $\C\subseteq \mathbb{R}^{d^\prime}$ be a non-empty, closed convex domain 
and $h : \mathbb{R}^{d^\prime} \rightarrow \mathbb{R}$ be a continuously differentiable function.
\begin{definition}
A function $h$ is $L$-Lipschitz continuous (or $L$-Lipschitz) on $\C$ if for all $\theta, \theta^\prime \in \C$, 
there exists a constant $L > 0$ such that
\begin{align*}
\abs{h(\theta) - h(\theta^\prime)} \leq L \normt{\theta - \theta^\prime}.
\end{align*}
\end{definition}
\begin{definition}
A function $h$ is $\beta$-smooth on $\C$ if its gradient is $\beta$-Lipschitz such that
\begin{align*}
\normt{\nabla h(\theta) - \nabla h(\theta^\prime)} \leq \beta \normt{\theta - \theta^\prime}.
\end{align*}
\end{definition}
\begin{definition} 
A function $h$ is $\alpha$-strongly convex on $\C$ if for all $\theta, \theta^\prime \in \C$, 
there exists a constant $\alpha > 0$ such that
\begin{align*}
\inner{ \nabla h(\theta) - \nabla h(\theta^\prime), \theta - \theta^\prime } \geq \alpha \normt{\theta - \theta^\prime}^2.
\end{align*}
\end{definition}

The above definition can be extended to functions map into matrix space.
For example, suppose $A : \mathbb{R}^{d^\prime} \rightarrow \mathbb{R}^{d \times d}$ is a function of 
$\theta \in \C$ and assume that all $\partial A_{j,k} / \partial \theta_i$ 's exist and are continuous. 

\begin{definition}
A function $A(\theta)$ is $L_A$-Lipschitz with respect to $\normf{\cdot}$ 
if for all $\theta, \theta^\prime \in \C$, there exists a constant $L_A > 0$ such that
\begin{align*}
\normf{  A(\theta)- A(\theta^\prime)} \leq L_A \normt{\theta - \theta^\prime}.
\end{align*}
Similarly, $A(\theta)$ is $L_{\mathtt{nuc}}$-Lipschitz with respect to $\normnuc{\cdot}$ (matrix nuclear norm)
there exists a constant $L_{\mathtt{nuc}} > 0$ such that
\begin{align*}
\normnuc{  A(\theta)- A(\theta^\prime)} \leq L_{\mathtt{nuc}} \normt{\theta - \theta^\prime}.
\end{align*}
\end{definition}
\begin{definition}
Let $A : \mathbb{R^{d^\prime}} \rightarrow \SM$ be a continuously differentiable function of $\theta \in \C$. 
If $A(\theta)$ is $\beta_A$-smooth if for all $\theta, \theta^\prime \in \C$, there exists a constant $\beta_A > 0$ such that 
\begin{align*}
\normf{ \frac{\partial A(\theta)}{\partial \theta} - \frac{\partial A(\theta^\prime)}{\partial \theta}}
\leq
\beta_A \normt{\theta - \theta^\prime}.
\end{align*}
\end{definition}

\subsection{Proof of Theorem \ref{thm:optdist} : optimal degree distribution}
By adding $\sum_{j=1}^\infty \rho^{-2j} = 1/(\rho^2 - 1)$ in both sides of \eqref{eq:degprob}, 
the optimization \eqref{eq:degprob} is equivalent to 
\begin{align} \label{eq:problem}
\min_{\{q_n\}_{n=0}^\infty} \ \sum_{j=1}^\infty \frac{\rho^{-2j}}{1 - \sum_{n=0}^{j-1} q_n}
\quad 
\text{subject to} \ \ &\ \sum_{n=1}^\infty n q_n = N, \ \sum_{n=0}^\infty q_n = 1 \ \ \text{and} \ \ q_n \geq 0. 
\end{align}
Note that the equality conditions can be written as
\begin{align} \label{eq:eqcond}
N = \sum_{n=1}^\infty n q_n 
= \sum_{n=1}^\infty \sum_{j=1}^n q_n
= \sum_{j=1}^\infty \sum_{n=j}^\infty q_n
= \sum_{j=1}^\infty \left( 1 - \sum_{n=0}^{j-1} q_n \right).
\end{align}
By Cauchy-Schwarz inequality for infinite series, we have
\begin{align*}
N \sum_{j=1}^\infty \frac{\rho^{-2j}}{1 - \sum_{n=0}^{j-1} q_n}
&= \left(\sum_{j=1}^\infty \left( 1 - \sum_{n=0}^{j-1} q_n \right) \right) \left( \sum_{j=1}^\infty \frac{\rho^{-2j}}{1 - \sum_{n=0}^{j-1} q_n}\right) \\
&\geq \left( \sum_{j=1}^\infty \rho^{-j} \right)^2 = \frac1{(\rho-1)^2}
\end{align*}
and the equality holds when 
$q_0 = 1 - N(\rho-1)\rho^{-1}$ and $q_n = N(\rho-1)^2 \rho^{-(n+1)}$ for $n \geq 1$.
However, this solution is not feasible when a given integer $N$ is greater than $\frac{\rho}{\rho - 1}$ (due to $q_0 < 0$).
The solution of \eqref{eq:problem} exists since the feasible region is closed and nonempty. For example, 
\begin{align} \label{eq:optq}
q_{n}^* = 
\begin{dcases}
0 \quad &\text{for } \ 0 \leq n \leq k, \\
1 - \frac{(N-k-1)(\rho-1)}{\rho} \quad &\text{for } \ n = k+1, \\
\frac{(N - k - 1)(\rho-1)^2}{\rho^{n-k}} \quad &\text{for } \ k+2 \leq n \\
\end{dcases}
\end{align}
with $k := N-1-\lfloor \frac{\rho}{\rho-1}\rfloor$ is feasible and achieves the objective function of \eqref{eq:problem}
\begin{align*}
\frac{1 - \rho^{-2(k+1)}}{\rho^2-1} + \frac{1}{(N-k-1)(\rho-1)^2 \rho^{2(k+1)}}.
\footnotemark
\end{align*}
\footnotetext{If $k=-1$, it is equivalent to the minimum from Cauchy-Schwarz inequality.}To figure out that $q^*$ is the optimal solution, 
one can investigate KKT conditions of \eqref{eq:problem}.
However, in general, KKT theorem can not be applied to infinite dimensional problems.
Instead, we consider the finite dimensional approximation of \eqref{eq:problem}:
\begin{align} \label{eq:finite}
\min_{q_0,\dots,q_T} &\sum_{j=1}^T \frac{\rho^{-2j}}{1 - \sum_{n=0}^{j-1} q_n}
 \quad
\text{subject to} \quad \sum_{n=0}^T n q_n = N, \sum_{n=0}^T q_n = 1 \quad \text{and} \quad q_n \geq 0. 
\end{align}
As we show in later, 
one can obtain the optimal solution of \eqref{eq:finite} for sufficently large $T$ using KKT conditions, which is
\begin{align} \label{eq:solq}
q_{n} = 
\begin{dcases}
0 \quad &\text{for } \ 0 \leq n \leq k, \\
1 - \frac{(N-k-1)(\rho-1)}{1 - \rho^{-T+k+1}}\rho^{-1} \quad &\text{for } \ n = k+1, \\
\frac{(N - k - 1)(\rho-1)^2}{1 - \rho^{-T+k+1}}  \rho^{-n +k} \quad &\text{for } \ k+2 \leq n \leq T-1 \\
\frac{(N - k - 1)(\rho-1)}{\rho^{T-k-1}-1}  \quad &\text{for } \ n = T \\
\end{dcases}
\end{align}
with $k := N-1-\lfloor \frac{\rho}{\rho-1}\rfloor$ and achieves the minimum 
\begin{align} \label{eq:optfinite}
\frac{1 - \rho^{-2(k+1)}}{\rho^2-1} + \frac{\left(1 - \rho^{-T+k+1}\right)^2}{(N-k-1)(\rho-1)^2 \rho^{2(k+1)}}.
\end{align}

We will show that the minimum of the infinite problem \eqref{eq:problem} is equivalent to the limit of \eqref{eq:optfinite} 
(a similar approach was introduced in \cite{schochetman1992finite}).
We first extend $q_n$ to the point with infinite dimension.

Let ${q}^{(T)}_n = q_n$ for $n \leq T$ and ${q}^{(T)}_n = 0$ for $n > T$,
then ${q}^{(T)}=(q_0^{(T)}, q_1^{(T)}, \dots)$ is a feasible point of \eqref{eq:problem}.
Note that $\lim_{T \rightarrow \infty} q^{(T)}_n = q^*_n$ for all $n$. Define that 
\begin{align*}
f(q, T) = 
\begin{dcases}
&\sum_{j=1}^T \frac{\rho^{-2j}}{1 - \sum_{n=0}^{j-1} q_n} := C(q;T), \qquad T = 1,2,\dots, \\
&\sum_{j=1}^\infty \frac{\rho^{-2j}}{1 - \sum_{n=0}^{j-1} q_n} := C(q), \quad \qquad T = \infty
\end{dcases}
\end{align*}
for $q = (q_0, q_1, \dots )$. 
We claim that $f$ is continuous. 
Suppose $T_i \in \mathbb{N}$ is a nondecreasing infinite sequence such that $T_i > k, T_i \rightarrow \infty$ and $q^{(T_i)} \rightarrow  q^*$ as $i \rightarrow \infty$.
Consider that
\begin{align} \label{eq:thm3_cont}
\left| f(q^*, \infty) - f(q^{(T_i)}, T_i) \right|
&= \left| C(q^*) - C(q^{(T_i)}; T_i) \right| \nonumber \\
&= \left| \sum_{j=1}^\infty \frac{\rho^{-2j}}{1 - \sum_{n=0}^{j-1} q^*_n}
- \sum_{j=1}^{T_i} \frac{\rho^{-2j}}{1 - \sum_{n=0}^{j-1} q^{(T_i)}_n} \right| \nonumber \\
&\leq \left| \sum_{j=T_i+1}^\infty \frac{\rho^{-2j}}{1 - \sum_{n=0}^{j-1} q^*_n} \right|
+ \left| \sum_{j=1}^{T_i} {\rho^{-2j}}\left(\frac1{1 - \sum_{n=0}^{j-1} q^{*}_n} - \frac1{1 - \sum_{n=0}^{j-1} q^{(T_i)}_n}\right) \right| \nonumber\\
&\leq \frac{\rho^{-T_i-k}}{(N-k-1)(\rho-1)}
+ \frac{\rho^{-T_i-k-1}}{(N-k-1)(\rho-1)^2}
\end{align}
and \eqref{eq:thm3_cont} goes to zero as $i \rightarrow \infty$. 
In addition, the feasible set of \eqref{eq:finite} is nondecreasing, i.e., if we define the feasible regions as
\begin{align*}
X(T) &:= \left\{q : \sum_{n=0}^T n q_n = N, \sum_{n=0}^T q_n = 1, q_n \geq 0, q_n = 0 \quad \text{for} \ n > T\right\}, \\
X &:=\left\{q : \sum_{n=0}^\infty n q_n = N, \sum_{n=0}^\infty q_n = 1, q_n \geq 0\right\}
\end{align*}
then $X(T) \subseteq X(T+1)$ for any $T$.
This leads to $\lim_{T \rightarrow \infty} X(T) = \cup_{T \geq 1} X(T) = X$.
Therefore, by the Berge's Maximum Theorem \cite{berge1963topological}, the minimum of the finite dimensional problem \eqref{eq:optfinite} 
converges to that of infinite problem \eqref{eq:problem}, i.e.,
\begin{align*}
\min\left\{
\sum_{j=1}^\infty \frac{\rho^{-2j}}{1 - \sum_{n=0}^{j-1} q_n}
: q \in X 
\right\}
&= \lim_{T\rightarrow \infty} 
\min\left\{
\sum_{j=1}^T \frac{\rho^{-2j}}{1 - \sum_{n=0}^{j-1} q_n}
: q \in X(T) 
\right\} \\
&= \lim_{T\rightarrow \infty} 
\left( \frac{1 - \rho^{-2(k+1)}}{\rho^2-1} + \frac{\left(1 - \rho^{-T+k+1}\right)^2}{(N-k-1)(\rho-1)^2 \rho^{2(k+1)}}\right) \\
&= \frac{1 - \rho^{-2(k+1)}}{\rho^2-1} + \frac{1}{(N-k-1)(\rho-1)^2 \rho^{2(k+1)}}.
\end{align*}
Since $q^*$ in \eqref{eq:optq} achieves the above minimum,
it follows that $q^*$ in \eqref{eq:optq} is the minimizer of \eqref{eq:problem}.

The remaining part is to obtain the solution of the finite dimensional approximation \eqref{eq:finite} using KKT conditions.
Since the objective and all inequality conditions are {\it convex} functions, 
any feasible solution that satisfies KKT conditions are optimal.
Define the Lagrangian as
\begin{align*}
\mathcal{L}(q, \lambda, \nu)
= \sum_{j=1}^T \frac{\rho^{-2j}}{1 - \sum_{n=0}^{j-1} q_n}
+\lambda_1\left( \sum_{n=0}^T n q_n - N\right)
+ \lambda_2 \left( \sum_{n=0}^T q_n - 1\right)
- \sum_{n=0}^T \nu_n q_n
\end{align*}
where $\lambda_1,\lambda_2$ and $\nu_0, \dots, \nu_T$ are the Lagrangian multipliers of 
equality and inequality condition, respectively.
The corresponding KKT conditions are following:
\begin{itemize}
\item {\bf Stationary:} For $0 \leq n \leq T$,
\begin{align} \label{eq:kkt1}
\frac{\partial \mathcal{L}}{\partial q_n}
= \sum_{j=n+1}^{T} \frac{\rho^{-2j}}{(1 - \sum_{n^\prime=0}^{j-1} q_{n^\prime})^2}
+ \lambda_1 n + \lambda_2
- \nu_n = 0, \tag{C1}
\end{align}
\item {\bf Primal feasibility:} 
\begin{align} \label{eq:kkt2}
\sum_{n=0}^T n q_n = N,\  \sum_{n=0}^T q_n = 1, \ q_n \geq 0, \tag{C2}
\end{align}
\item {\bf Dual feasibility:} For $0 \leq n \leq T$,
\begin{align} \label{eq:kkt3}
\nu_n \geq 0, \tag{C3}
\end{align}
\item {\bf Complementary slackness:} For $0 \leq n \leq T$,
\begin{align} \label{eq:kkt4}
\nu_n q_n = 0. \tag{C4}
\end{align}
\end{itemize}

Consider $(q, \lambda, \nu)$ that satisfies the KKT conditions holds that 
$\nu_n = 0$, $k+1\leq n \leq T$ and $\nu_n \neq 0$, $0 \leq n \leq k$ for some $k \in [0,T]$.
By the complementary slackness \eqref{eq:kkt4}, $q_0=q_1=\dots=q_k = 0$.
Substracting two consecutive stationary conditions \eqref{eq:kkt1}, we obtain for $0 \leq n\leq T-1$
\begin{align} \label{eq:sub}
\frac{\partial \mathcal{L}}{\partial q_{n}} - \frac{\partial \mathcal{L}}{\partial q_{n+1}}
= \frac{\rho^{-2(n+1)}}{\left( 1 - \sum_{n^\prime = 0}^{n} q_{n^\prime}\right)^2} 
- \lambda_1 - \nu_n + \nu_{n+1}
= 0, 
\end{align}
which implies that 
\begin{align} \label{eq:case2}
1 - \sum_{n^\prime=0}^n q_{n^\prime} = \frac{\rho^{-(n+1)}}{\sqrt{\lambda_1}}\quad \text{ for } k+1 \leq n \leq T-1.
\end{align}
Putting them together into the equality condition \eqref{eq:eqcond} gives
\begin{align*}
N = \sum_{n=0}^{T-1} \left( 1 - \sum_{n^\prime=0}^n q_{n^\prime}\right) 
= k+1 + \frac{1 - \rho^{-(T-k-1)}}{\sqrt{\lambda_1} \rho^{k+1}(\rho - 1)},
\end{align*}
equivalently, $\sqrt{\lambda_1} = \frac{1 - \rho^{-(T-k-1)}}{(N-k-1) \rho^{k+1}\left(\rho-1\right)}$.
Therefore, we obtain the solution from \eqref{eq:case2}:
\begin{align*} 
q_{n} = 
\begin{dcases}
0 \quad &\text{for } \ 0 \leq n \leq k, \\
1 - \frac{(N-k-1)(\rho-1)}{1 - \rho^{-T+k+1}}\rho^{-1} \quad &\text{for } \ n = k+1, \\
\frac{(N - k - 1)(\rho-1)^2}{1 - \rho^{-T+k+1}}  \rho^{-n +k} \quad &\text{for } \ k+2 \leq n \leq T-1 \\
\frac{(N - k - 1)(\rho-1)}{\rho^{T-k-1}-1}  \quad &\text{for } \ n = T \\
\end{dcases}
\end{align*}
In order to satisfy the primal feasibility \eqref{eq:kkt2}, 
it should hold that
\begin{align} \label{eq:k1}
\frac{\rho (1 - \rho^{-T+k+1})}{\rho-1} \geq N-k-1 \quad \text{and} \quad k \leq N-1.
\end{align}

From \eqref{eq:sub}, the dual variables $\nu$ can be written as for $n \leq k$
\begin{align*}
\nu_{n} - \nu_{n+1} &= 
{\rho^{-2(n+1)}}
- \lambda_1
\end{align*}
and in order to satisfy the dual feasibility \eqref{eq:kkt3}, i.e., $\nu_n > 0$ for $n \leq k$,
the sufficient condition is
\begin{align} \label{eq:k2}
N-k-1 > \frac{(1-\rho^{-T+k+1})}{\rho - 1}.
\end{align}
To satisfy both \eqref{eq:k1} and \eqref{eq:k2}, there exists an integer in the interval $\left[\frac{(1-\rho^{-T+k+1})}{\rho - 1},\frac{\rho(1-\rho^{-T+k+1})}{\rho - 1} \right]$.
We now choose $T$ large enough such that 
\begin{align*}
\left\lfloor \frac{\rho}{\rho-1}\right\rfloor \leq \frac{\rho(1 - \rho^{-T+N})}{\rho - 1},
\end{align*}
and it holds that $\left\lfloor \frac{\rho}{\rho-1}\right\rfloor \in \left[\frac{(1-\rho^{-T+k+1})}{\rho - 1},\frac{\rho(1-\rho^{-T+k+1})}{\rho - 1} \right]$ for some $0 \leq k \leq N-1$.
By choosing $k := N-1-\lfloor \frac{\rho}{\rho-1}\rfloor$,
$\{q_n\}_{n=0}^T$ in \eqref{eq:solq} satisfies the KKT conditions and acheives the minimum
\begin{align*}
\frac{1 - \rho^{-2(k+1)}}{\rho^2-1} + \frac{\left(1 - \rho^{-T+k+1}\right)^2}{(N-k-1)(\rho-1)^2 \rho^{2(k+1)}}.
\end{align*}

\subsection{Proof of Theorem \ref{thm:sgd} : convergence analysis of SGD}
We recall that $\theta^{(t)} \in \C \subseteq \mathbb{R}^{d^\prime}$ by the parameter in the $t$-th iteration 
and $\theta^{(t)}_i$ by its element $i$-th position for $i = 1, \dots, d^\prime$.
For simplicity, we denote that
\begin{align*}h(\theta) := \Sigma_f(A(\theta)) + g(\theta)\end{align*}
and $\theta^*\in \C$ be the optimal of $h$.
Let $\psi^{(t)}$ be our unbiased gradient estimator for $\Sigma_f(A(\theta))$ using $\{\vv^{(k)}\}_{k=1}^M$ and $n$, that is, 
\begin{align*}
\mathbf{E}_{n,\vv}[\psi^{(t)}] = \frac{\partial}{\partial \theta} \Sigma_f(A(\theta))
\end{align*}
and $\nabla g^{(t)}$ be the derivative of $g(\theta)$ at $\theta^{(t)}$.
Unless stated otherwise, we use $\norm{\cdot}$ as the entry-wise $L_2$-norm, i.e., 
$L_2$-norm for vectors and Frobenius norm for matrices.
Now we are ready to show the convergence guarantee for SGD.
The iteration of SGD can be written as
\begin{align*}
\theta^{(t+1)} = \proj{\theta^{(t)} - \eta ( \psi^{(t)} + \nabla g^{(t)})}
\end{align*}
where $\proj{\cdot}$ is the projection mapping in $\C$.
The remaining part is similar with 
standard proof of the projected stochastic gradient descent.
First, we write the error between $\theta^{(t)}$ and $\theta^*$ as
\begin{align*}
\norm{\theta^{(t+1)} - \theta^*}^2 
&= \norm{\Pi_{\C}(\theta^{(t)} - \eta ( \psi^{(t)} + \nabla g^{(t)} )) - \theta^*}^2 \nonumber \\
&\leq \norm{\theta^{(t)} - \eta ( \psi^{(t)} + \nabla g^{(t)} ) - \theta^*}^2 \nonumber \\
&= \norm{\theta^{(t)} - \theta^*}^2 - 2\eta \inner{\psi^{(t)}+\nabla g^{(t)}, \theta^{(t)} - \theta^*} + \eta^2 \norm{\psi^{(t)}+\nabla g^{(t)}}^2 \\
&\leq 
\norm{\theta^{(t)} - \theta^*}^2 - 2\eta \inner{\psi^{(t)}+\nabla g^{(t)}, \theta^{(t)} - \theta^*} 
+ 2 \eta^2 \norm{\psi^{(t)}}^2 +2\eta^2 \norm{\nabla g^{(t)}}^2 \\
&\leq
\norm{\theta^{(t)} - \theta^*}^2 - 2\eta \inner{\psi^{(t)}+\nabla g^{(t)}, \theta^{(t)} - \theta^*} 
+ 2 \eta^2 \norm{\psi^{(t)}}^2 +2\eta^2 L_g^2
\end{align*}
where the inequality in the second line holds from the convexity of $\C$, 
the inequality in the fourth line follows from that $\norm{a + b}^2 \leq 2 \norm{a}^2 + 2 \norm{b}^2$
and the last inequality follows from Lipschitz continuity of $g$.
Taking the expectation with respect to random samples (i.e., random degree and vectors) in $t$-th iteration, which denoted as $\mathbf{E}_t[\cdot]$, we have 
\begin{align} \label{eq:converge0}
\mathbf{E}_t[\norm{\theta^{(t+1)} - \theta^*}^2]
&\leq 
\norm{\theta^{(t)} - \theta^*}^2 - 2\eta \inner{\nabla h( \theta^{(t)}), \theta^{(t)} - \theta^*} + 4 \eta^2 B^2
\end{align}
where $B^2 := \max \left( \mathbf{E}_t[\norm{\psi^{(t)}}^2], L_g^2\right)$.
In addition, by $\alpha$-strong convexity of $h$, it holds that
\begin{align}\label{eq:convF}
\alpha \norm{ \theta^{(t)} - \theta^*}^2 \leq \inner{\nabla h( \theta^{(t)}), \theta^{(t)} - \theta^*}.
\end{align}
Combining \eqref{eq:converge0} with \eqref{eq:convF} and taking the expectation on both sides with respect to all random samples from $1,...,t$ iteration, we obtain that
\begin{align*}
\mathbf{E}[\norm{\theta^{(t+1)} - \theta^*}^2]
\leq 
(1-2\eta \alpha) \mathbf{E}[\norm{\theta^{(t)} - \theta^*}^2] + 4 \eta^2 B^2
\end{align*}
Applying $\eta = \frac1{\alpha t}$, we have 
\begin{align*}
\mathbf{E}[\norm{\theta^{(t+1)} - \theta^*}^2]
\leq
\left( 1 - \frac2{t}\right) \mathbf{E}[\norm{\theta^{(t)} - \theta^*}^2] + \frac{4 B^2}{\alpha^2 t^2}.
\end{align*}
Therefore, if 
$\mathbf{E}[\norm{\theta^{(1)} - \theta^*}^2] \leq 4B^2 / \alpha^2$ holds, then the result follows by induction on $t\geq 1$.
Under assumption that $\mathbf{E}[\norm{\theta^{(t)} - \theta^*}^2] \leq {4 B^2}/{(\alpha^2 t)}$, it is straightforward that 
\begin{align*}
\mathbf{E}[\norm{\theta^{(t+1)} - \theta^*}^2]
\leq
\left(1 - \frac2t\right) \frac{4B^2}{\alpha^2 t} + \frac{4B^2}{\alpha^2 t^2}
\leq \frac{4B^2}{\alpha^2} \left( \frac1{t+1}\right).
\end{align*}
To show the case of $t=1$, we recall the strong convexity of $h$ and use Cauchy-Schwartz inequality:
\begin{align*}
\alpha \norm{\theta^{(1)} - \theta^*}^2 \leq \inner{ \psi^{(1)} + \nabla g^{(1)}, \theta^{(1)} - \theta^*}
\leq \norm{\psi^{(1)} + \nabla g^{(1)}} \norm{\theta^{(1)} - \theta^*},
\end{align*}
which leads to that 
\begin{align*}
\alpha^2 \mathbf{E}[\norm{\theta^{(1)} - \theta^*}^2] 
\leq \mathbf{E}[\norm{\psi^{(1)} + \nabla g^{(1)}}^2]
\leq 4 B^2.
\end{align*}
Recall that Lemma \ref{lmm:varbound} implies that for all $t$
\begin{align*}
\mathbf{E}_{t} [\norm{\psi^{(t)}}^2] \leq
\left({2 L_A^2}/{M} + d^\prime L_{\mathtt{nuc}}^2\right)
\left(C_1 + {C_2 N^4}{\rho^{-2N}}\right).
\end{align*}
for some constants $C_1, C_2 >0$. This completes the proof of Theorem \ref{thm:sgd}.
\subsection{Proof of Theorem \ref{thm:svrg} : convergence analysis of SVRG}
Denote the objective as $h(\theta) := \Sigma_f(A(\theta)) + g(\theta)$.
Let 
$\psi^{(t)},\widetilde{\psi}$ be our unbiased gradient estimator for $\Sigma_f(A(\theta))$ at $\theta^{(t)}$ and $\widetilde{\theta}^{(s)}$, respectively, and $\widetilde{\mu} = \nabla \Sigma_f(A(\widetilde{\theta}^{(s)}))$.
We use $\nabla g^{(t)}$ by the exact gradient of $g(\theta)$ at $\theta^{(t)}$, which is easy to compute.
The iteration of SVRG can be written as
\begin{align*}
\theta^{(t+1)} = \Pi_{\C}(\theta^{(t)} - \eta \xi^{(t)}),
\quad \text{where} \quad
\xi^{(t)}:=\psi^{(t)} - \widetilde{\psi} + \widetilde{\mu} + \nabla g^{(t)}
\end{align*}
where $\proj{\cdot}$ is the projection mapping in $\C$.
We first introduce the 
lemma that implies our unbiased estimator is $\beta$-smooth for some $\beta>0$.
\begin{lemma} \label{lmm:smooth}
Suppose that assumptions $(\mathcal{A}0)$-$(\mathcal{A}2)$ 
hold and assume that $A: \C \rightarrow \SM$ is $\beta_A$-smooth function with respect to $\normf{\cdot}$.
Let $\psi, \psi^\prime$ be our unbiased gradient estimator \eqref{eq:estder2} at $\theta, \theta^\prime \in \C \subseteq \mathbb{R}$ 
using the same $\{\vv^{(k)}\}_{k=1}^M$ and $n$ (drawn from \eqref{eq:optdist} with mean $N$).
Then, it holds that
\begin{align*}
\mathbf{E}_{n,\vv}\left[
\norm{\psi + \nabla g(\theta) - \psi^\prime - \nabla g(\theta^\prime)}_2^2
\right] 
\leq 
\left(2 \beta_g^2 +  \left(\frac{L_A^4 + \beta_A^2}{M} + L_A^4 \right)\left(D_1 + \frac{D_2 N^8}{\rho^{2N}} \right)\right)\norm{\theta - \theta^\prime}^2_2.
\end{align*}
where $D_1,D_2>0$ are some constants independent of $M,N$.
\end{lemma}
The proof of the above lemma is given in Section \ref{sec:lmms}. For notational simplicity, we denote 
\begin{align*}
\beta^2 := 2\beta_g^2 + \left(\frac{L_A^4 + \beta_A^2}{M} + L_A^4 \right)\left(D_1 + \frac{D_2 N^8}{\rho^{2N}} \right).
\end{align*}

The remaining part mimics the analysis of \cite{garber2015fast}.
Using the above lemma, the moment of the gradient estimator is bounded as
\begin{align} \label{eq:varboundsvrg}
\mathbf{E}_t[\norm{\psi^{(t)} - \widetilde{\psi} + \widetilde{\mu} + \nabla g^{(t)}}^2]
&\leq 2 \mathbf{E}_t[\norm{\psi^{(t)} + \nabla g^{(t)} - \psi^{*} - \nabla g^{*}}^2]  
+ 2 \mathbf{E}_t[\norm{\widetilde{\psi} - \psi^{*} - \nabla g^{*} - \widetilde{\mu}}^2] \nonumber \\
&\leq 2 \mathbf{E}_t[\norm{\psi^{(t)} + \nabla g^{(t)} - \psi^{*} - \nabla g^{*}}^2]  
+ 2 \mathbf{E}_t[\norm{\widetilde{\psi} + \nabla \widetilde{g} - \psi^{*} - \nabla g^{*}}^2] \nonumber \\
&\leq 2 \beta^2 \left(\norm{\theta^{(t)} - \theta^*}^2 + \norm{\widetilde{\theta} - \theta^*}^2\right)
\end{align}
where the inequality in the first line holds from $\norm{a + b}^2 \leq 2 (\norm{a}^2 + \norm{b}^2)$,
the inequality in the second line holds that $\mathbf{E}[ \norm{ X - \mathbf{E}[X] }^2] \leq \mathbf{E}[ \norm{X}^2]$ 
for any random variable $X$ 
and the last inequality holds from Lemma \ref{lmm:smooth}.

Now, we use similar procedures of Theorem \ref{thm:sgd} to obtain 
\begin{align*}
\norm{\theta^{(t+1)} - \theta^*}^2 &= \norm{\proj{\theta^{(t)} - \eta \xi^{(t)}}- \theta^*}^2 \\
&\leq \norm{\theta^{(t)} - \eta \xi^{(t)}- \theta^*}^2 \\
&= \norm{\theta^{(t)} - \theta^*}^2 - 2 \eta \inner{\theta^{(t)} - \theta^*, \xi_t} + \norm{\xi_t}^2.
\end{align*}
where the inequality holds from the convexity of $\C$. 
Taking the expectation with respect to random samples of $t$-th iteration, which denoted as $\mathbf{E}_t[\cdot]$, we obtain that 
\begin{align*}
\mathbf{E}_t [\norm{\theta^{(t+1)} - \theta^*}^2] 
&= \norm{\theta^{(t)} - \theta^*}^2 - 2 \eta \inner{\theta^{(t)} - \theta^*, \nabla h(\theta^{(t)})} + \eta^2 \mathbf{E}_t[ \norm{\xi_t}^2] \\
&\leq \norm{\theta^{(t)} - \theta^*}^2 - 2 \eta \alpha \norm{\theta^{(t)} - \theta^*}^2 + \eta^2 \mathbf{E}_t[ \norm{\xi_t}^2] \\
&\leq \norm{\theta^{(t)} - \theta^*}^2 - 2 \eta \alpha \norm{\theta^{(t)} - \theta^*}^2 + 
 2 \eta^2 \beta^2 \left(\norm{\theta^{(t)} - \theta^*}^2 + \norm{\widetilde{\theta} - \theta^*}^2\right)
\end{align*}
where the inequality in the second line holds from the $\alpha$-strong convexity of the objective 
and the last inequality holds from \eqref{eq:varboundsvrg}.
Taking the expectation over the randomness of all iterations, we have
\begin{align*}
\mathbf{E} [\norm{\theta^{(t+1)} - \theta^*}^2] 
- \mathbf{E} [\norm{\theta^{(t)} - \theta^*}^2]
&\leq 2 \eta \left(\eta \beta^2 - \alpha \right) \mathbf{E}[\norm{\theta^{(t)} - \theta^*}^2] + 2 \eta^2 \beta^2 \mathbf{E}[\norm{\widetilde{\theta} - \theta^*}^2]
\end{align*}
Summing both sides over $t=1,2,\dots, T$, it yields that
\begin{align*}
\mathbf{E} [\norm{\theta^{(T)} - \theta^*}^2] 
- \mathbf{E} [\norm{\theta^{(0)} - \theta^*}^2]
&\leq 2 \eta \left(\eta \beta^2 - \alpha \right) \sum_{t=0}^{T-1} \mathbf{E}[\norm{\theta^{(t)} - \theta^*}^2] + 2 T \eta^2 \beta^2 \mathbf{E}[\norm{\widetilde{\theta} - \theta^*}^2]
\end{align*}

Rearranging and using the facts that $\mathbf{E} [\norm{\theta^{(T)} - \theta^*}^2]  \geq 0$ and $\widetilde{\theta} = \widetilde{\theta}^{(s)}$, we get
\begin{align*}
2 \eta \left(\alpha - \eta \beta^2\right) \sum_{t=0}^{T-1} \mathbf{E}[\norm{\theta^{(t)} - \theta^*}^2]
\leq 
\left(1 + 2 T \eta^2 \beta^2\right) \mathbf{E}[\norm{\theta^{(0)} - \theta^*}^2].
\end{align*}
From $\widetilde{\theta}^{(s+1)} = \frac1T \sum_{t=1}^T \theta^{(t)}$ and Jensen's inequality, we have 
\begin{align*}
\mathbf{E}[ \norm{\widetilde{\theta}^{(s+1)} - \theta^*}^2]
\leq 
\frac1T \sum_{t=1}^T \mathbf{E}[ \norm{\theta^{(t)} - \theta^*}^2]
\leq
\frac{1 + 2 T \eta^2 \beta^2}{2 \eta T \left(\alpha - \eta \beta^2\right)} \mathbf{E}[\norm{\widetilde{\theta}^{(s)} - \theta^*}^2]
\end{align*}
Substituting $\eta = \frac{\alpha}{7\beta^2}$ and $T \geq \frac{49\beta^2}{2\alpha^2}$, we have that 
\begin{align*}
\mathbf{E}[\norm{\widetilde{\theta}^{(S)} - \theta^*}^2]
\leq
r^S
\mathbf{E}[\norm{\widetilde{\theta}^{(0)} - \theta^*}^2]
\end{align*}
for some $0 < r < 1$.

\subsection{Proof of lemmas} \label{sec:lmms}

\subsubsection{Proof of Lemma \ref{lmm:unbiased}} \label{sec:lmm:unbiased}
Without loss of generality, we choose $a=-1,b=1$.
An analytic function $f$ has an (unique) infinite Chebyshev series expansion:
$f(x) = \sum_{j=0}^\infty b_j T_j(x).$
and recall that our proposed estimator as
\begin{align*}
\wpn{x} = \sum_{j=0}^n \frac{b_j}{1 - \sum_{i=0}^{j-1} q_i} T_j(x).
\end{align*}
To prove that $\mathbf{E}_n\left[ \wpn{x}\right] = f(x)$, we define two sequences:
\begin{align*}
A_{M} \coloneqq \sum_{j=0}^{M} \sum_{n=j}^{M} q_n \frac{b_j T_j(x)}{1 - \sum_{i=0}^{j-1} q_i}, \quad
B_{M,K} \coloneqq \sum_{j=0}^{M} \sum_{n=j}^{K} q_n \frac{b_j T_j(x)}{1 - \sum_{i=0}^{j-1} q_i}.
\end{align*}
Then, it is easy to show that
\begin{align*}
\lim_{M \to \infty} A_M 
= \sum_{j=0}^\infty \sum_{n=j}^\infty q_n\frac{b_j T_j(x)}{1 - \sum_{i=0}^{j-1} q_i} 
= \sum_{n=0}^\infty q_n \left(\sum_{j=0}^n \frac{b_j T_j(x)}{1 - \sum_{i=0}^{j-1} q_i} \right) 
= \sum_{n=0}^\infty q_n \wpn{x} = \mathbf{E}_n\left[ \wpn{x}\right],
\end{align*}
and
\begin{align*}
\lim_{M \to \infty} \lim_{K \to \infty} B_{M,K} 
= \lim_{M \to \infty} \sum_{j=0}^{M} \left( \sum_{n=j}^{\infty} q_n\right) \frac{b_j T_j(x)}{1 - \sum_{i=0}^{j-1} q_i}
= \lim_{M \to \infty} \sum_{j=0}^{M} b_j T_j(x) = f(x).
\end{align*}
In general, $A_M$ and $B_{M,K}$ might not converge to the same values. 
Now, consider sufficiently large $K \geq M$. From the condition that $\lim_{n \to \infty} \sum_{i=n+1}^\infty q_i \wpn{x}$, we have
\begin{align*}
\mathbf{E}_n\left[ \wpn{x}\right] - f(x) &= \lim_{M \to \infty} \lim_{K \to \infty}  \left( A_M - B_{M,K} \right) 
= \lim_{M \to \infty} \lim_{K \to \infty} \left(\sum_{j=0}^M \sum_{n=M+1}^K q_n \frac{b_j T_j(x)}{1 - \sum_{i=0}^{j-1} q_i}\right) \\
&= \lim_{M \to \infty} \lim_{K \to \infty} \left( \sum_{n=M+1}^K q_n \right) \left( \sum_{j=0}^M \frac{b_j T_j(x)}{1 - \sum_{i=0}^{j-1} q_i} \right)\\
&= \lim_{M \to \infty} \left( \sum_{n=M+1}^\infty q_n \right) \left( \sum_{j=0}^M \frac{b_j T_j(x)}{1 - \sum_{i=0}^{j-1} q_i} \right) \\
&= \lim_{M \to \infty} \left( \sum_{n=M+1}^\infty q_n \right) \widehat{p}_M(x) = 0.
\end{align*}
Therefore, we can conclude that $\wpn{x}$ is an unbiased estimator of $f(x)$.
In addition, this also holds for the trace of matrices due to its linearity:
$\mathbf{E}_n\left[ \tr{\wpn{A}} \right] = \tr{f(A)}.$
By taking expectation over Rademacher random vectors $\vv$ and degree $n$,
we establish the unbiased estimator of spectral-sums:
\begin{align*}
\mathbf{E}_{n,\vv}\left[ \vv^\top \wpn{A} \vv \right]
=\mathbf{E}_{n}\left[\mathbf{E}_\vv\left[ \vv^\top \wpn{A} \vv | n\right]\right]
=\mathbf{E}_{n}\left[ \tr{ \wpn{A} }\right]
=\tr{f(A)},
\end{align*}
For fixed $\mathbf{v}$ and ${n}$, the function 
$h(\theta) \coloneqq \vv^\top \hat{p}_n(A(\theta)) \vv$ is a linear
combination of all entries of $A$, so the fact that all 
partial derivatives $\partial A_{j,k} /\partial \theta_i$
exist and are continuous implies that the partial derivatives of $h$
with respect to $\theta_1,\dots,\theta_{d^\prime}$ exist and are continuous.
In particular, since expectation over $\vv \in [-1,+1]^d$ is a finite sum, 
it is straightforward that the gradient operator and expectation operator can be interchanged:
\begin{align*}
\nabla_\theta \tr{f(A)} 
&= \nabla_\theta \E{\vv^\top \wpn{A} \vv} 
= \E{\nabla_\theta \vv^\top \wpn{A} \vv}.
\end{align*}
In the case of trace probing vector $\vv$ is a continuous random vector, i.e., Gaussian, 
we turn to use the Leibniz rule which allows to interchange the gradient operator and expectation operator. 
Hence, we conclude the same result.
This completes the proof of Lemma \ref{lmm:unbiased}.

\subsubsection{Proof of Lemma \ref{lmm:2}} \label{sec:lmm:2}
Without loss of generality, we choose $a=-1,b=1$. 
We first introduce the orthogonality of Chebyshev polynomials of the first kind, that is,
\begin{align*}
\int_{-1}^1 \frac{T_i(x) T_j(x)}{\sqrt{1-x^2}} dx = 
\begin{cases}
0 \qquad & i\neq j,\\
\pi \qquad & i=j=0, \\
\frac{\pi}{2} \qquad & i=j\neq 0.
\end{cases}
\end{align*}
Given functions $f,g$ defined on $[-1,1]$, Chebyshev induced inner-product and weighted norm are defined as
\begin{align*}
\inner{f,g}_C = \int_{-1}^1 \frac{f(x) g(x)}{\sqrt{1-x^2}} dx, 
\qquad 
\norm{f}_C^2 = \inner{f,f}_C.
\end{align*}

For a fixed $n$, the square of Chebyshev weighted error can be written as
\begin{align*}
\norm{\widehat{p}_n - f}_C^2 &= \norm{\widehat{p}_n - p_n + p_n - f}_C^2 
= \norm{p_n - f}_C^2 + 2 \inner{p_n - f , \widehat{p}_n - p_n}_C + \norm{\widehat{p}_n - p_n}_C^2 \\
&\overset{(\dagger)}{=} \norm{p_n - f}_C^2 + \norm{\widehat{p}_n - p_n}_C^2 \\
&= \left\|\sum_{j=n+1}^\infty b_j T_j\right\|_C^2 + \left\|\sum_{j=1}^n \frac{\sum_{k=0}^{j-1}q_n}{1 - \sum_{k=0}^{j-1}q_n} b_j T_j\right\|_C^2 \\
&\overset{(\ddagger)}{=} 
\frac{\pi}{2} \sum_{j=n+1}^\infty b_j^2 + 
\frac{\pi}{2} \sum_{j=1}^n \left( \frac{\sum_{i=0}^{j-1}q_i}{1 - \sum_{i=0}^{j-1}q_i} b_j \right)^2.
\end{align*}
Both the second equality $(\dagger)$ and the last equality $(\ddagger)$ 
come from the orthogonality of Chebyshev polynomials and the following facts:\
\begin{align*}
p_n - f \ &: \text{linear combination of } \ T_{n+1}(x), T_{n+2}(x), \cdots, \\
\widehat{p}_n - p_n \ &: \text{linear combination of } \ T_0(x), \cdots , T_n(x).
\end{align*}
The Chebyshev weighted variance can be computed by taking expectation with respect to $n$:
\begin{align*}
\frac{2}{\pi} \mathbf{E}_n [\norm{\widehat{p}_n - f}_C^2 ] &= \frac{2}{\pi} \sum_{n=0}^\infty q_n \norm{\widehat{p}_n - f}_C^2 
= q_0 \sum_{j=1}^\infty b_j^2 
+ \sum_{n=1}^\infty q_n \left(  \sum_{j=1}^n \left(\frac{b_j\sum_{i=0}^{j-1} q_i}{1 - \sum_{i=0}^{j-1} q_i} \right)^2 + \sum_{j=n+1}^\infty b_j^2 \right) \\
&= \sum_{j=1}^\infty b_j^2 \left( q_0 + \sum_{i=1}^{j-1} q_i + \left( \frac{\sum_{i=0}^{j-1} q_i}{1 - \sum_{i=0}^{j-1} q_i}\right)^2 \sum_{i=j}^\infty q_i \right) \\
&= \sum_{j=1}^\infty b_j^2 \left( \sum_{i=0}^{j-1} q_i + \left( \frac{\sum_{i=0}^{j-1} q_i}{1 - \sum_{i=0}^{j-1} q_i}\right)^2 \left( 1 - \sum_{i=0}^{j-1} q_i \right) \right) \\
&= \sum_{j=1}^\infty b_j^2 \left( \sum_{i=0}^{j-1} q_i +  \frac{\left(\sum_{i=0}^{j-1} q_i\right)^2}{1 - \sum_{i=0}^{j-1} q_i} \right) 
= \sum_{j=1}^\infty b_j^2 \left( \frac{\sum_{i=0}^{j-1} q_i}{1 - \sum_{i=0}^{j-1} q_i} \right).
\end{align*}
This completes the proof of Lemma \ref{lmm:2}.

\subsubsection{Proof of Lemma \ref{lmm:varbound}} \label{sec:lmm:varbound}
First, we define the $n$ degree Chebyshev polynomials of the first kind by $T_n(\cdot)$ and the second kind by $U_n(\cdot)$. One important property is that
$T_n^\prime (x) := \frac{d}{dx} T_n(x) = n U_{n-1}(x)$ for $n\geq 1$ (see \cite{mason2002chebyshev}). 
Consider our unbiased estimator with a single random sample, i.e., a Rademacher vector $\vv$ and a degree $n$ drawn from the optimal distribution \eqref{eq:optdist}.

From the intermediate result \eqref{eq:lmm4_main} in the proof of Lemma \ref{lmm:amort}, the gradient estimator can be written as following:
\begin{align} \label{eq:vGv}
\psi_i := \frac{\partial}{\partial \theta_i} \vv^\top \wpn{A(\theta)} \vv 
= \frac{2}{b-a}\vv^\top G \vv
\end{align}
where
\begin{align*}
G = \sum_{j=0}^{n-1} \widehat{b}_{j+1} \left( 2\sum_{r=0}^{j} {}^\prime T_r(\widetilde{A}) \dA U_{j-r}(\widetilde{A}) \right)
\end{align*}
and $\widetilde{A} = \frac{2}{b-a} A(\theta) - \frac{b+a}{b-a} I$ 
and $\sum {}^\prime$ implies the summation where the first term is halved.
We also note that $\tr{G} = \tr{\frac{\partial A}{\partial \theta_i} \widehat{p}_n^\prime({A})}$.
Here, our goal is to find the upper bound of $\mathbf{E}_{n,\vv}[ \psi_i^2 ]$, that is, 
\begin{align*}
\frac{(b-a)^2}{4} \mathbf{E}_{n,\vv}[ \psi_i^2 ]
=
\mathbf{E}_{n,\vv}\left[ \left( \vv^\top G \vv \right)^2 \right]
=
\mathbf{E}_{n}\left[\mathbf{E}_{\vv}\left[ \left( \vv^\top G \vv \right)^2 \big| n \right]\right].
\end{align*}
From \cite{hutchinson1989stochastic,avron2011randomized}, we have that 
$\mathrm{Var}_{\vv}[\vv^\top A \vv] = 2 (\normf{A}^2 -  \sum_{i=1}^d A_{ii}^2)\leq 2 \normf{A}^2$ 
and $\mathbf{E}_\vv[\vv^\top G \vv] = \tr{G}$ for Rademacher random vector $\vv \in [-1,1]^d$ and $A\in \SM$.
Therefore, we have 
\begin{align} \label{eq:Gvar}
\mathbf{E}_{\vv}\left[ \left(\vv^\top G \vv\right)^2 \big| n\right]
= \mathrm{Var}_\vv[ \vv^\top G \vv\big | n] + \mathbf{E}_\vv\left[\vv^\top G \vv \big | n \right]^2 
\leq 2 \normf{G}^2 + \left( \tr{G}\right)^2.
\end{align}

The first term in \eqref{eq:Gvar} is bounded as 
\begin{align*}
2\normf{G}^2 
&\leq 2 \normf{\dA}^2 \left( \sum_{j=1}^{n} \abs{\widehat{b}_{j}} \left(2 \sum_{r=0}^j {}^\prime \norm{T_r ( \widetilde{A})}_2 \norm{U_{j-r}(\widetilde{A})}_2\right)\right)^2 \\
&\leq 2 \normf{\dA}^2 \left( \sum_{j=1}^{n} \abs{\widehat{b}_{j}} \left(2 \sum_{r=0}^j {}^\prime (j - r + 1)\right) \right)^2\\
&= 2 \normf{\dA}^2 \left( \sum_{j=1}^{n} \abs{\widehat{b}_{j}} j^2 \right)^2
\end{align*}
which the first inequality comes from the triangle inequality of $\normf{\cdot}$ and
the fact that $\normf{X Y} \leq \normf{X} \norm{Y}_2$ for mutliplicable matrices $X$ and $Y$.
The inequality in the second line holds from $\norm{T_i(\widetilde{A})}_2 \leq 1$ 
and $\norm{U_i(\widetilde{A})}_2 \leq i+1$ for $i \geq 0$.

For second term in \eqref{eq:Gvar}, we use 
the inequality that $\tr{XY} \leq \normnuc{X}\norm{Y}_2$ for real symmetric matrices $X,Y$ (see Section \ref{sec:lmm:others}) to obtain
\begin{align*}
\left( \tr{G}\right)^2 = \left( \tr{\dA {\widehat{p}_n}^\prime({A})}\right)^2 
&\leq \normnuc{\dA}^2 \normt{{\widehat{p}_n}^\prime({A})}^2 \\
&= \normnuc{\dA}^2 \left\| \sum_{j=1}^{n} \widehat{b}_j j U_{j-1}(\widetilde{A})\right\|_2^2 \\
&\leq \normnuc{\dA}^2 \left( \sum_{j=1}^{n} \abs{\widehat{b}_{j}} j^2 \right)^2
\end{align*}
where 
the equality in the second line uses that $\left(\sum_{j=0}^n \widehat{b}_j T_j(x) \right)^\prime = \sum_{j=1}^{n} \widehat{b}_j j U_{j-1}(x)$
and the last inequality holds from $\norm{U_i(\widetilde{A})}_2 \leq i+1$.
Putting all together into \eqref{eq:Gvar}
and summing for all $i = 1, \dots, d^\prime$, 
we obtain that
\begin{align*}
\mathbf{E}_{n,\vv}[\psi^2]  = \sum_{i=1}^{d^\prime}\mathbf{E}_{n,\vv} [ \psi_i^2] 
&\leq \frac{4}{(b-a)^2}\sum_{i=1}^{d^\prime} \mathbf{E}_{n}\left[2 \normf{G}^2 + \left( \tr{G}\right)^2 \right]\\
&\leq \frac{4}{(b-a)^2}\sum_{i=1}^{d^\prime}\left(2\normf{\dA}^2 + \normnuc{\dA}^2 \right)\mathbf{E}_{n}\left[\left( \sum_{j=1}^{n} \abs{\widehat{b}_{j}} j^2 \right)^2\right]\\
&\leq \frac{4}{(b-a)^2} \left(2\normf{\frac{\partial A}{\partial \theta}}^2 + \sum_{k=1}^{d^\prime}\normnuc{\dA}^2 \right)
\mathbf{E}_{n}\left[\left( \sum_{j=1}^{n} \abs{\widehat{b}_{j}} j^2 \right)^2\right].
\end{align*}
When we estimate $\psi$ using $M$ Rademacher random vectors $\{\vv^{(k)}\}_{k=1}^M$, the variance in \eqref{eq:Gvar} is reduced by $1/M$.
Hence, we have
\begin{align*}
\mathbf{E}_{n,\vv}[\psi^2] 
&\leq \frac4{(b-a)^2} \left(\frac2M\normf{\frac{\partial A}{\partial \theta}}^2 + \sum_{k=1}^{d^\prime}\normnuc{\dA}^2 \right)
\mathbf{E}_{n}\left[\left( \sum_{j=1}^{n} \abs{\widehat{b}_{j}} j^2 \right)^2\right]\\
&\leq \frac4{(b-a)^2} \left(\frac{2 L_A^2}{M} + d^\prime L_{\mathtt{nuc}}^2\right)
\mathbf{E}_{n}\left[\left( \sum_{j=1}^{n} \abs{\widehat{b}_{j}} j^2 \right)^2\right].
\end{align*}

Finally, we introduce the following lemma to bound the right-hand side, 
where its proof is given in Section \ref{sec:lmm:moment2}. 

\begin{lemma} \label{lmm:moment2}
Suppose that $q^*_n$ is the optimal degree distribution as defined in \eqref{eq:optdist} and 
$b_j$ is the Chebyshev coefficients of analytic function $f$. 
Define the weighted coefficient $\widehat{b}_j$ as
$\widehat{b}_j= {b_j}/({1-\sum_{i=0}^{j-1} q^*_i})$
for $j\geq 0$ and conventionally $q^*_{-1} = 0$. Then, there exists constants $C_1, C_2 > 0$ independent of $M,N$ such that
\begin{align*}
\sum_{n=1}^\infty q^*_n \left( \sum_{j=1}^n | \widehat{b}_j | j^2 \right)^2 \leq C_1 + \frac{C_2 N^4}{\rho^{2N}}.
\end{align*}
\end{lemma}
To sum up, we conclude that
\begin{align*}
\mathbf{E}_{n,\vv}[\psi^2] 
&\leq \left(\frac{2 L_A^2}{M} + d^\prime L_{\mathtt{nuc}}^2\right)
\left( C_1 + \frac{C_2 N^4}{\rho^{2N}}\right)
\end{align*}
for some constant $C_1, C_2 > 0$.
This completes the proof of Lemma \ref{lmm:varbound}.

\subsubsection{Proof of Lemma \ref{lmm:amort}} \label{sec:lmm:amort}

We consider more general case in which $A \in \SM$ is a function of parameter $\theta = [\theta_1, \dots, \theta_{d^\prime}]$, and our goal is to derive a closed form of $\frac{\partial}{\partial \theta_i} \vv^\top \pn{A} \vv$ with allowing only vector operations.
We begin by observing that for any polynomial $p_n$ and symmetric matrix $A \in \SM$,
the derivative of $\mathbb{E}_{\vv}[ \vv^\top \pn{A} \vv]$ can be expressed by a simple formulation, that is, 
$$
\frac{\partial}{\partial \theta_i} \mathbb{E}_{\vv}[ \vv^\top \pn{A} \vv]
=
\frac{\partial}{\partial \theta_i} \tr{p_n(A)} = p_n'(A) \dA.$$
However, it does not holds that
\begin{align*}
\frac{\partial}{\partial \theta_i} \vv^\top p_n(A) \vv 
= \frac{\partial}{\partial \theta_i} \tr{p_n(A) \vv \vv^\top }
\neq p_n'(A) \vv \vv^\top \frac{\partial A}{\partial \theta_i}.
\end{align*}
for some vector $\vv \in \mathbb{R}^d$. 
This is because of 
$\frac{\partial}{\partial \theta_i} \tr{A^j \vv \vv^\top} \neq j A^{j-1} \vv \vv^\top \dA$ in general.

If $p_n(x)$ is the truncated Chebyshev series, i.e., $p_n(x) = \sum_{j=0}^n b_j T_j(x)$,
we can compute $\frac{\partial}{\partial \theta_i} \vv^\top p_n(A)\vv$ efficiently using the recursive relation of Chebyshev polynomials, that is, 
\begin{align*}
T_{j+1}(x) = 2xT_j(x) - T_{j-1}(x),
\end{align*}
where $T_j(x)$ is the Chebyshev polynomial of the first-kind with degree $j$.
Let $\ww_j := T_j(A)\vv$ for $j\geq 0$, and we have that
\begin{align} \label{eq:lmm4_dpn}
\frac{\partial}{\partial \theta_i} \vv^\top p_n(A)\vv
&= 
\frac{\partial}{\partial \theta_i}\left(\sum_{j=0}^n b_j \vv^\top T_j(A)\vv \right)
=
\sum_{j=0}^n b_j \vv^\top \left(\frac{\partial}{\partial \theta_i}  T_j(A)\vv \right) 
= \sum_{j=0}^n b_j \vv^\top \frac{\partial \ww_j}{\partial \theta_i}.
\end{align}
In the right hand side,
$\vv^\top (\frac{\partial \ww_j}{\partial \theta_i})$ 
can be computed using the recursion $\ww_{j+1} = 2 A \ww_j - \ww_{j-1}$:
\begin{align*}
\vv^\top \frac{\partial \ww_{j+1}}{\partial \theta_i}
&= 
\vv^\top \frac{\partial}{\partial \theta_i} \left( 2A \ww_j - \ww_{j-1}\right) 
= 2 \ \vv^\top \frac{\partial}{\partial \theta_i} \left({A \ww_j}\right)  - \vv^\top \frac{\partial \ww_{j-1}}{\partial \theta_i} \\
&= 2 \ \left( \vv^\top \frac{\partial A}{\partial \theta_i} \ww_{j} + \vv^\top A \frac{\partial \ww_j}{\partial \theta_i} \right)  - 
\vv^\top \frac{\partial \ww_{j-1}}{\partial \theta_i}
\end{align*}
where $\vv^\top \frac{\partial \ww_{1}}{\partial \theta_i}=\vv^\top \frac{\partial A}{\partial \theta_i} \vv$ and $\vv^\top \frac{\partial \ww_{0}}{\partial \theta_i} = 0.$
Applying induction on $j\geq 1$, we can obtain that 
\begin{align} \label{eq:lmm4_dw}
\vv^\top \frac{\partial \ww_{j+1}}{\partial \theta_i}
= \sum_{k=0}^j \left(2 - \mathds{1}_{k=0}\right)\ww_k^\top \frac{\partial A}{\partial \theta_i} \yy_{j-k},
\end{align}
where $\yy_{j+1}= 2A\yy_j - \yy_{j-1} = 2\ww_{j+1} + \yy_{j-1}, \yy_1 = 2A \vv$ and $\yy_0 = \vv$. 
\footnote{
Indeed, $\yy_j = U_j(A) \vv$ for $j \geq 1$, where $U_j(x)$ is the $j$-th Chebyshev polynomial of the second-kind.
}
Putting \eqref{eq:lmm4_dw} to \eqref{eq:lmm4_dpn}, we get
\begin{align} \label{eq:lmm4_main}
\frac{\partial}{\partial \theta_i} \vv^\top p_n(A) \vv &= \sum_{j=0}^{n-1} b_{j+1} \vv^\top \frac{\partial \ww_{j+1}}{\partial \theta_i}
= \sum_{j=0}^{n-1} b_{j+1} \left( \sum_{k=0}^{j}\left(2 - \mathds{1}_{k=0}\right)\ww_k^\top \frac{\partial A}{\partial \theta_i} \yy_{j-k}\right).
\end{align}
In case when $A = \theta \theta^\top + \varepsilon I$ and $\theta \in \mathbb{R}^{d \times r}$,
it holds that for $\ell=1, \dots, d$ and $m=1,\dots,r$,
\begin{align} \label{eq:lmm4_dtheta}
\frac{\partial A}{\partial \theta_{\ell,m}} = \mathbf{e}_\ell \theta_{:,m}^\top + \theta_{:,m} \mathbf{e}_\ell^\top,
\end{align}
where $\theta_{:,m} \in \mathbb{R}^d$ is the $m$-th column of $\theta$ and 
$\mathbf{e}_{\ell} \in \mathbb{R}^d$ is a unit vector with the index $\ell$.
Finally, we substitute \eqref{eq:lmm4_dtheta} to \eqref{eq:lmm4_main} to have
\begin{align*}
\left[\frac{\partial}{\partial \theta} \vv^\top p_n(A) \vv\right]_{\ell, m}
&=
\frac{\partial}{\partial \theta_{\ell, m}} \vv^\top p_n(A) \vv
= \sum_{j=0}^{n-1} b_{j+1} \left( \sum_{k=0}^{j}\left(2 - \mathds{1}_{k=0}\right)\ww_k^\top \left(\frac{\partial A}{\partial \theta_{\ell,m}} \right)\yy_{j-k}\right) \\
&=
\sum_{j=0}^{n-1} b_{j+1} \left( \sum_{k=0}^{j}\left(2 - \mathds{1}_{k=0}\right)\ww_k^\top \left(\mathbf{e}_\ell \theta_{:,m}^\top + \theta_{:,m} \mathbf{e}_\ell^\top\right) \yy_{j-k}\right) \\
&\stackrel{(\dagger)}{=}
\sum_{j=0}^{n-1} b_{j+1} \left( \sum_{k=0}^{j}\left(2 - \mathds{1}_{k=0}\right) \left(\mathbf{e}_\ell^\top \ww_k \yy_{j-k}^\top \theta_{:,m} + \mathbf{e}_\ell^\top \yy_{j-k} \ww_k^\top \theta_{:,m} \right) \right)\\
&=\mathbf{e}_\ell^\top  \Bigg[
\sum_{j=0}^{n-1} b_{j+1} \left( \sum_{k=0}^{j}\left(2 - \mathds{1}_{k=0}\right) \left(\ww_k \yy_{j-k}^\top  + \yy_{j-k} \ww_k^\top \right) \right)\Bigg] \theta_{:,m} \\
&\stackrel{(\ddagger)}{=}
\mathbf{e}_\ell^\top  \Bigg[
 \sum_{j=0}^{n-1} b_{j+1} \left( \sum_{k=0}^{j}\left(2 - \mathds{1}_{k=0}\right) 2 \ \ww_k \yy_{j-k}^\top \right)\Bigg] \theta_{:,m} \\
&=\mathbf{e}_\ell^\top  \Bigg[
2 \sum_{j=0}^{n-1} b_{j+1} \left( \sum_{k=0}^{j}\left(2 - \mathds{1}_{k=0}\right) \ww_k \yy_{j-k}^\top \right) \theta \Bigg]  \mathbf{e}^\prime_m \\
&= \Bigg[
2 \sum_{k=0}^{n-1} \left(2 - \mathds{1}_{k=0}\right) \ww_k \Bigg( \sum_{j=k}^{n-1} b_{j+1} \yy_{j-k}\Bigg)^\top \theta \Bigg]_{\ell, m}
\end{align*}
where $\mathbf{e}^\prime_m \in \mathbb{R}^r$ is the unit vector with index $m$ satisfying with $\theta_{:,m} = \theta \mathbf{e}^\prime_m$.
The equality $(\dagger)$ holds from that $\mathbf{a}^\top \mathbf{b} = \mathbf{b}^\top \mathbf{a}$ for any two vectors $\mathbf{a}$ and $\mathbf{b}$, and for the equality $(\ddagger)$ it is easy to check that
$
\sum_{k=0}^{j}\left(2 - \mathds{1}_{k=0}\right) \ww_k \yy_{j-k}^\top
=
\sum_{k=0}^{j}\left(2 - \mathds{1}_{k=0}\right) \yy_{j-k} \ww_k^\top
$
using $2 \ww_j = \yy_j - \yy_{j-2}$ for $j \geq 2$. Thus, 
\begin{align*}
\nabla_\theta \vv^\top p_n(A) \vv
=
2 \sum_{k=0}^{n-1} \left(2 - \mathds{1}_{k=0}\right) \ww_k \Bigg( \sum_{j=k}^{n-1} b_{j+1} \yy_{j-k}\Bigg)^\top \theta
\end{align*}
This completes the proof of Lemma \ref{lmm:amort}.

\subsubsection{Proof of Lemma \ref{lmm:smooth}}
The proof of Lemma \ref{lmm:smooth} is similar with the proof of Lemma \ref{lmm:varbound}.
We recall the formulation 
\begin{align*}
\psi_i := \frac{\partial}{\partial \theta_i} \vv^\top \wpn{A(\theta)} \vv 
= \frac{2}{b-a}\vv^\top G \vv
\end{align*}
where
\begin{align*}
G = \sum_{j=0}^{n-1} \widehat{b}_{j+1} \left( 2\sum_{r=0}^{j} {}^\prime T_r(\widetilde{A}) \dA U_{j-r}(\widetilde{A}) \right)
\end{align*}
and $\widetilde{A} = \frac{2}{b-a} A(\theta) - \frac{b+a}{b-a} I$. 
Define that $\Delta G := G(\theta) - G(\theta^\prime)$. 
Our goal is to find some $\beta \in \mathbb{R}$ such that $\mathbf{E}_{n,\vv}[(\vv^\top \Delta G\vv)^2]\leq \beta^2 (\theta_i - \theta_i^\prime)^2$.
For notational simplicity, 
we write that
\begin{align*}
\Delta T_r &:= T_r(\widetilde{A}) - T_r(\widetilde{A}^\prime) = T_r - T_r^\prime, \qquad \Delta U_j := U_j(\widetilde{A}) - U_j(\widetilde{A}^\prime) = U_j - U_j^\prime, \\
\Delta A &:= \frac{2}{b-a} \left(A(\theta) - A(\theta^\prime) \right), \ \ \
\Delta \frac{\partial A}{\partial \theta} := \frac{\partial A(\theta)}{\partial \theta} - \frac{\partial A(\theta^\prime)}{\partial \theta}, \ \ \
\Delta \theta = \theta - \theta^\prime.
\end{align*}
and $\Delta G$ can be expressed as
\begin{align*}
\Delta G 
&= \sum_{j=0}^{n-1} \widehat{b}_{j+1} \left( 2 \sum_{r=0}^j {}^\prime T_r \dA U_{j-r} - T_r^\prime \dA^\prime U_{j-r}^\prime\right).
\end{align*}

We use similar procedure in the proof of Lemma \ref{lmm:varbound} to obtain 
\begin{align} \label{eq:delG}
\frac{(b-a)^2}{4}\mathbf{E}_{n,\vv}\left[ \left( \psi_i - \psi_i^\prime \right)^2 \right]
&=\mathbf{E}_{n,\vv}\left[ \left( \vv^\top \Delta G \vv \right)^2 \right]
=
\mathbf{E}_{n}\left[\mathbf{E}_{\vv}\left[ \left( \vv^\top \Delta G \vv \right)^2 \big| n \right]\right] \nonumber \\
&= 
\mathbf{E}_n \left[ \mathrm{Var}_\vv[ \vv^\top \Delta G \vv\big | n] + \mathbf{E}_\vv\left[\vv^\top \Delta G \vv \big | n \right]^2 \right] \nonumber \\
&\leq
\mathbf{E}_n \left[ 2 \normf{\Delta G}^2 + \left( \tr{\Delta G}\right)^2 \right].
\end{align}

For the first term in \eqref{eq:delG}, we use the triangle inequality to obtain
\begin{align*}
\normf{\Delta G} 
\leq 
\sum_{j=0}^{n-1} \abs{\widehat{b}_{j+1}} \left( 2 \sum_{r=0}^j {}^\prime 
\underbrace{\left\|T_r \dA U_{j-r} - T_r^\prime \dA^\prime U_{j-r}^\prime\right\|_F}_{(\ddagger)}\right)
\end{align*}
and consider that
\begin{align*} 
(\ddagger) &\leq \normf{\left(T_r - T_r^\prime\right) \frac{\partial A}{\partial \theta_i} U_{j-r}} 
+ \normf{T_r \frac{\partial A}{\partial \theta_i}\left(U_{j-r} - U_{j-r}^\prime\right)} 
+ \normf{T_r^\prime \left(\dA - \dA^\prime\right)U_{j-r}^\prime} \\
&\leq 
\normt{\Delta T_r} \normf{\frac{\partial A}{\partial \theta_i}} \normt{ U_{j-r}}
+
\normt{T_r} \normf{\frac{\partial A}{\partial \theta_i}} \normt{\Delta U_{j-r}}
+\normt{T_r^\prime} \normf{\Delta\dA} \normt{ U_{j-r}} \\
&\leq
\normt{\Delta A} r^2  \normf{\frac{\partial A}{\partial \theta_i}} (j-r+1) +  \normf{\frac{\partial A}{\partial \theta_i}}\frac{(j-r)(j-r+1)(j-r+2)}3 \normt{\Delta A} + \normf{\Delta \frac{\partial A}{\partial \theta_i}} (j-r+1)
\end{align*}
where the first inequality is from the triangle inequality of $\normf{\cdot}$ and the second inequality holds from 
$\normf{XY} \leq \normt{X}\normf{Y}$ for multiplicable matrices $X,Y$ and the last is from $\norm{T_i(\widetilde{A})}_2 \leq 1$, $\norm{U_i(\widetilde{A})}_2 \leq i+1$ for $i \geq 0$ 
and
\begin{align} \label{eq:dU}
\normt{U_i(X+E) - U_i(X)} \leq \frac{i(i+1)(i+2)}3 \normt{E}
\end{align}
for $X,E \in \SM$ satisfying with $\normt{X+E}, \normt{X} \leq 1$ (see Section \ref{sec:lmm:others}).

Summing $(\ddagger)$ for all $r = 0, 1, \dots, j$, we have 
\begin{align*}
\normf{\Delta G} 
&\leq \sum_{j=0}^{n-1} \abs{\widehat{b}_{j+1}}
\left(\normt{\Delta A} \normf{\frac{\partial A}{\partial \theta_i}} \frac{j(j+1)^2(j+2)}{3}+
\normf{\Delta \frac{\partial A}{\partial \theta_i}} (j+1)^2\right) \\
&\leq \max\left(\normt{\Delta A} \normf{\frac{\partial A}{\partial \theta_i}},\normf{\Delta \frac{\partial A}{\partial \theta_i}}\right) 
\sum_{j=0}^{n-1} \abs{\widehat{b}_{j+1}} \left( \frac{j(j+1)^2(j+2)}{3} + (j+1)^2\right) \\
&\leq \frac12 \max\left(\normt{\Delta A} \normf{\frac{\partial A}{\partial \theta_i}},\normf{\Delta \frac{\partial A}{\partial \theta_i}}\right) 
\sum_{j=0}^{n-1} \abs{\widehat{b}_{j+1}} {(j+1)^4}.
\end{align*}
If one estimates $\psi$ and $\psi^\prime$ using $M$ Rademacher random vectors, the variance of $\vv^\top \Delta G\vv$ is reduced by $1/M$ so that we have
\begin{align*}
2 \normf{\Delta G}^2 
\leq 
\frac1{2M} \max\left(\normt{\Delta A}^2 \normf{\frac{\partial A}{\partial \theta_i}}^2,\normf{\Delta \frac{\partial A}{\partial \theta_i}}^2\right) 
\left( \sum_{j=1}^{n} \abs{\widehat{b}_{j}} {j^4}\right)^2
\end{align*}

For the second term in \eqref{eq:delG}, it holds that
\begin{align*}
\tr{ \Delta G} 
= \tr{ \dA\left( \widehat{p}_n^\prime(A) - \widehat{p}_n^\prime(A^\prime) \right)}
&\leq \normf{\dA} \left\|\widehat{p}_n^\prime(A) - \widehat{p}_n^\prime (A^\prime)\right\|_F \\
&\leq \normf{\dA} \sum_{j=1}^n \abs{\widehat{b}_{j}} j \left\| U_{j-1}(\widetilde{A}) - U_{j-1}(\widetilde{A}^\prime)\right\|_F \\
&\leq \normf{\dA} \left\| \Delta A\right\|_F \sum_{j=1}^n \abs{\widehat{b}_{j}} \frac{(j^2-1)j^2}{3} \\
&\leq \normf{\dA} \frac{\left\| \Delta A\right\|_F}{3} \sum_{j=1}^n \abs{\widehat{b}_{j}} j^4.
\end{align*}
where 
the inequality in the first line holds from matrix version Cauchy-Schwarz inequality,
the inequality in the second line holds from 
$\widehat{p}_n^\prime(x) = \left(\sum_{j=0}^n \widehat{b}_j T_j(x) \right)^\prime = \sum_{j=1}^{n} \widehat{b}_j j U_{j-1}(x)$ 
and inequality in the third line holds from \eqref{eq:dU}.

Putting all together into \eqref{eq:delG}, we obtain that
\begin{align*}
&\mathbf{E}_{n,\vv}\left[ \left( \psi_i - \psi_i^\prime \right)^2 \right] 
= \mathbf{E}_n \left[ 2 \normf{\Delta G}^2 + \left(\tr{\Delta G}\right)^2\right]\\
&\leq 
\Bigg(
\frac1{2M}\max\left(\normt{\Delta A}^2 \normf{\frac{\partial A}{\partial \theta_i}}^2,\normf{\Delta \frac{\partial A}{\partial \theta_i}}^2\right) + 
\normf{\dA}^2 \frac{\normf{\Delta A}^2}{9}
\Bigg)
\mathbf{E}_n\left[\left( \sum_{j=1}^n \abs{\widehat{b}_{j}} j^4\right)^2\right] \\
&\leq
\Bigg(
\left( \frac1{2M} + \frac19 \right) \normf{\dA}^2
\normf{\Delta A}^2 + 
\frac1{2M}\normf{\Delta \frac{\partial A}{\partial \theta_i}}^2
\Bigg)
\mathbf{E}_n\left[\left( \sum_{j=1}^n \abs{\widehat{b}_{j}} j^4\right)^2\right] \\
&\leq
\left( \left( \frac1{2M} + \frac19 \right) \normf{\dA}^2 \frac{4 L_A^2 \normt{\Delta \theta}^2}{(b-a)^2}  + 
\frac1{2M}\normf{\Delta \frac{\partial A}{\partial \theta_i}}^2
\right)
\mathbf{E}_n\left[\left( \sum_{j=1}^n \abs{\widehat{b}_{j}} j^4\right)^2\right] 
\end{align*}
where the inequality in the second line holds from $\max(a,b) \leq a + b$ for $a,b\in \mathbb{R}^+$ 
and the inequality in the third line holds from the Lipschitz continuity on $A$ (assumption $\mathcal{A}(2)$), formally, 
\begin{align*}
\norm{A(\theta) - A(\theta^\prime)}_2
\leq 
\norm{A(\theta) - A(\theta^\prime)}_F \leq L_A \normt{\theta - \theta^\prime}.
\end{align*}

Summing the above for all $i = 1,2, \dots, d^\prime$
and using that $\normf{{\partial A}/{\partial \theta}} \leq L_A$ and 
$\normf{\Delta ( \partial A / \partial \theta )} \leq \beta_A \normt{\Delta \theta}$,
we get
\begin{align*}
\mathbf{E}_{n,\vv}\left[ \norm{\psi - \psi^\prime}_2^2\right]
\leq
D_0 \left(\frac{L_A^4 + \beta_A^2}{M} + L_A^4 \right) \normt{\Delta \theta}^2
\mathbf{E}_n\left[\left( \sum_{j=1}^n \abs{\widehat{b}_{j}} j^4\right)^2\right]
\end{align*}
for some constant $D_0>0$.

To bound the right-hand side, we introduce the following lemma, whose proof is in Section \ref{sec:lmm:moment4}.
\begin{lemma} \label{lmm:moment4}
Suppose that $q^*_n$ is the optimal degree distribution as defined in \eqref{eq:optdist} and 
$b_j$ is the Chebyshev coefficients of analytic function $f$. 
Define the weighted coefficient $\widehat{b}_j$ as
$\widehat{b}_j= {b_j}/({1-\sum_{i=0}^{j-1} q^*_i})$
for $j\geq 0$ and conventionally $q^*_{-1} = 0$. Then, there exists constants $D_1^\prime, D_2^\prime > 0$ independent of $M,N$ such that
\begin{align*}
\sum_{n=1}^\infty q^*_n \left( \sum_{j=1}^n | \widehat{b}_j | j^4 \right)^2 \leq D_1^\prime + \frac{D_2^\prime N^8}{\rho^{2N}}.
\end{align*}
\end{lemma}

Therefore, we obtain the result that
\begin{align} \label{eq:betapsi}
\mathbf{E}_{n,\vv}\left[ \norm{\psi - \psi^\prime}_2^2\right] \leq \beta^2 \normt{\theta - \theta^\prime}^2
\end{align}
where
\begin{align*}
\beta^2 := \left( \frac{L_A^4 + \beta_A^2}{M} + L_A^4 \right) \left( D_1 + \frac{D_2 N^8}{\rho^{2N}}\right)
\end{align*}

Under the assumption that $g(\theta)$ is $\beta_g$-smooth function (assumptio $\mathcal{A}(2)$,), we have that
\begin{align} \label{eq:betag}
\normt{ \nabla g(\theta) - \nabla g(\theta^\prime)}^2 \leq \beta_g^2 \normt{\theta - \theta^\prime}^2.
\end{align}
Summing both \eqref{eq:betapsi} and \eqref{eq:betag}, it yields that
\begin{align*}
\mathbf{E}_{n,\vv}\left[ \normt{\psi - \psi^\prime}^2 + \normt{\nabla g(\theta) - \nabla g(\theta^\prime)}^2\right]
\leq \left(\beta^2 + \beta_g^2\right) \normt{\theta - \theta^\prime}^2.
\end{align*}
Using $\norm{a + b} \leq 2(\norm{a}^2 + \norm{b}^2)$ again, we conclude that
\begin{align*}
\mathbf{E}_{n,\vv}\left[ \normt{\psi + \nabla g(\theta) - \psi^\prime  - \nabla g(\theta^\prime)}^2\right]
\leq 2 \left(\beta^2 + \beta_g^2\right) \normt{\theta - \theta^\prime}^2.
\end{align*}
This completes the proof of Lemma \ref{lmm:smooth}.

\subsubsection{Proof of Lemma \ref{lmm:moment2}} \label{sec:lmm:moment2}
Recall that the optimal degree distribution as 
\begin{align*}
q_i^{*} = 
\begin{dcases}
0 &\text{for } \ i < K \\
1 - {(N-K)\left(\rho-1\right)}{\rho^{-1}} &\text{for } \ i = K \\
{(N-K)(\rho-1)^2}{\rho^{-i-1+K}} &\text{for } \ i > K.
\end{dcases}
\end{align*}
where $K = \max(0, N - \lfloor \frac{\rho}{\rho-1} \rfloor)$.
We first use the upper bound on the coefficients from \eqref{eq:decayrate}, i.e., $\abs{{b}_j} \leq 2U / \rho^j$ to obtain
\begin{align} \label{eq:2moment}
\sum_{n=1}^\infty q^*_n \left( \sum_{j=1}^n | \widehat{b}_j | j^2 \right)^2
=
\sum_{n=K}^\infty q^*_n \left( \sum_{j=1}^n | \widehat{b}_j | j^2 \right)^2
\leq
4U^2\sum_{n=K}^\infty q^*_n \left( \sum_{j=1}^n \frac{j^2}{(1 - \sum_{i=0}^{j-1} q_i^*)\rho^j}\right)^2
\end{align}
To express \eqref{eq:2moment} more simple, we define that 
\begin{align*}
\Lambda := \sum_{j=1}^K
\frac{j^2}{(1 - \sum_{i=0}^{j-1} q_i^*)\rho^j}
= 
\sum_{j=1}^K \frac{j^2}{\rho^j}
\leq 
\frac{\rho(\rho+1)}{(\rho - 1)^3}
\end{align*}
which equals to the second term in the summation \eqref{eq:2moment} when $n=K$.
For $n \geq K + i, i \geq 1$, we get
\begin{align} \label{eq:lmm9_eq1}
\sum_{j=1}^{K+i} \frac{j^2}{(1 - \sum_{i=0}^{j-1} q_i^*)\rho^j}
= \Lambda + \frac{\sum_{j=1}^i (K+j)^2}{(N-K)(\rho-1)\rho^{K}}.
\end{align}
Putting $q_i^*$ and \eqref{eq:lmm9_eq1} to the right hand side of \eqref{eq:2moment}, we have
\begin{align*}
\left(1 - {(N-K)\frac{\rho-1}{\rho}}\right) &\Lambda^2
+ \left(N-K\right)\left(\frac{\rho-1}{\rho}\right)^2\left(\Lambda + \frac{(K+1)^2}{(N-K)(\rho-1)\rho^{K}}\right)^2 \\
&+ \left(N-K\right)\left(\frac{\rho-1}{\rho}\right)^2 \frac1\rho \left(\Lambda + \frac{\sum_{j=1}^2 (K+j)^2}{(N-K)(\rho-1)\rho^{K}}\right)^2 \\
&+ \left(N-K\right)\left(\frac{\rho-1}{\rho}\right)^2 \frac1{\rho^2} \left(\Lambda + \frac{\sum_{j=1}^3 (K+j)^2}{(N-K)(\rho-1)\rho^{K}}\right)^2 \\
&+ \cdots.
\end{align*}
Rearranging all terms with respect to $\Lambda$, we obtain that
\begin{align*}
\Lambda^2 &+ \frac{2(\rho-1)}{\rho^{K+1}} \left(\sum_{i=1}^\infty \frac{\sum_{j=1}^i (K+j)^2}{\rho^i} \right)\Lambda
+
\frac{1}{(N-K)\rho^{2K+1}} \left(\sum_{i=1}^\infty \frac{\left(\sum_{j=1}^i (K+j)^2\right)^2}{\rho^i} \right).
\end{align*}
Note that 
\begin{align*}
\sum_{i=1}^\infty \frac{\sum_{j=1}^i (K+j)^2}{\rho^i}
= 
\frac{K^2 \rho (\rho-1)^2 + 2K \rho^2(\rho-1) + \rho^2(\rho+1)}{(\rho-1)^4}
\end{align*}
and
\begin{align*}
\sum_{i=1}^\infty \frac{(\sum_{j=1}^i (K+j)^2)^2}{\rho^i}
= {\texttt{poly}(K^4)}.
\end{align*}
Since $K=O(N)$ and $N-K=O(1)$, we can conclude that
\begin{align*}
\sum_{n=1}^\infty q^*_n \left( \sum_{j=1}^n | \widehat{b}_j | j^2 \right)^2 \leq C_1 + C_2 \frac{N^4}{\rho^{2N}}
\end{align*}
for some constants $C_1,C_2> 0$ not depend on $N$.

\subsubsection{Proof of Lemma \ref{lmm:moment4}} \label{sec:lmm:moment4}
The proof of Lemma \ref{lmm:moment4} is straightforward from that of Lemma \ref{lmm:moment2}.
One can replace $j^2$ into $j^4$ in the proof of Lemma \ref{lmm:moment2}, which results in $N^8$ dependence.
We omit the details of the proof.

\subsubsection{Proof of other lemmas} \label{sec:lmm:others}

\begin{lemma} \label{lmm:cbbound}
Suppose that $A, A+E \in \mathbb{R}^{d \times d}$ are symmetric matrices and they have eigenvalues in $[-1,1]$. Let $T_i$ and $U_i$ be the first and the second kind of Chebyshev basis polynomial with degree $i \geq 0$, respectively. Then, it holds that 
\begin{align*}
\norm{T_i(A+E) - T_i(A)} \leq i^2 \norm{E}, \quad
\norm{U_i(A+E) - U_i(A)} \leq \frac{i(i+1)(i+2)}{3} \norm{E}.
\end{align*}
where $\norm{\cdot}$ can be $\norm{\cdot}_2$ (spectral norm) or $\normf{\cdot}$ (Frobenius norm).
\end{lemma}
\begin{proof}
Denote $R_i := T_{i} \left(A + E\right) - T_{i} \left( A\right)$.
From the recursive relation of Chebyshev polynomial, i.e., $T_{j+1}(x) = 2A T_j(x) - T_{j-1}(x)$,
$R_i$ has following property:
\begin{align*} 
&R_{i+1}
= 2\left( A + E\right) R_i - R_{i-1} + 2 E \ T_i\left( A\right)
\end{align*}
for $i\geq 1$ where $R_1 = E$, $R_0 = \mathbf{0}$.
By induction on $i$, it is easy to show that
\begin{align*}
R_{i+1} = 2
\sum_{j=0}^{i} {}^\prime U_{i-j} \left(A + E\right)  E \ T_{j} \left(A\right)
\end{align*}
where $U_j(x)$ is the Chebyshev polynomial of the second kind. Therefore, we have
\begin{align*}
\norm{R_{i+1}}_F 
&\leq 2 \sum_{j=0}^{i} {}^\prime \norm{U_{i-j}\left( A + E\right) E \ T_{j}(A) }_F \\
&\leq 2 \sum_{j=0}^{i} {}^\prime \normt{U_{i-j}\left( A + E\right)} \normf{E} \ \normt{T_{j}(A)} \\
&\leq 2 \sum_{j=0}^{i} {}^\prime (i+1-j) \normf{E}  = (i+1)^2 \normf{E}
\end{align*}
where the second inequality holds from $\norm{YX}_F = \norm{XY}_F \leq \norm{X}_2 \norm{Y}_F$ for matrices $X,Y$.
This also holds for $\normt{\cdot}$ giving that
$
\normt{R_{i+1}} \leq (i+1)^2 \normt{E}.
$
Similarly, we denote $Y_i := U_i(A+E) - U_i(A)$. By induction on $i$, it is easy to show that 
\begin{align*}
Y_{i+1} = 2 \sum_{j=0}^{i}  U_{i-j} \left(A + E\right)  E \ U_{j} \left(A\right)
\end{align*}
Then, we have that for $i\geq 0$
\begin{align*}
\norm{Y_{i+1}}_F 
&\leq 2 \sum_{j=0}^{i} \norm{U_{i-j}\left( A + E\right) E \ U_{j}(A) }_F \\
&\leq 2 \sum_{j=0}^{i} \normt{U_{i-j}\left( A + E\right)} \normf{E} \ \normt{U_{j}(A)} \\
&\leq 2 \sum_{j=0}^{i} (i+1-j) (j+1)\normf{E}  \\
&= \frac{(i+1)(i+2)(i+3)}3 \normf{E}.
\end{align*}
This also holds for $\normt{\cdot}$ giving that
$
\normt{Y_{i+1}} \leq \frac{(i+1)(i+2)(i+3)}3 \normt{E}.
$
This completes the proof of Lemma \ref{lmm:cbbound}.
\end{proof}

\begin{lemma} \label{lmm:trnuc}
For symmetric matrices $A, B \in \SM$, it holds that
$\tr{AB} \leq \norm{A}_{\mathtt{nuc}} \normt{B}$.
\end{lemma}
\begin{proof}
Since $A$ is real symmetric, it can be written as $A = \sum_{i=1}^d \lambda_i \mathbf{u}_i \mathbf{u}_i^\top$ where $\lambda_i$ and $\mathbf{u}_i$ is $i$-th eigenvalue and eigenvector, respectively. Then, the result follows that
\begin{align*}
\tr{AB}
= \sum_{i=1}^d \lambda_i \ \tr{\mathbf{u}_i \mathbf{u}_i^\top B}
&= \sum_{i=1}^d \lambda_i \ \mathbf{u}_i^\top B \mathbf{u}_i \\
&\leq \sum_{i=1}^d \abs{\lambda_i} \ \mathbf{u}_i^\top B \mathbf{u}_i \\
&\leq \sum_{i=1}^d \abs{\lambda_i} \norm{B}_2
= \normnuc{A} \norm{B}_2.
\end{align*}
This completes the proof of Lemma \ref{lmm:trnuc}.
\end{proof}



\end{document}